\documentclass[8pt]{article}
\usepackage{graphicx}
\usepackage{pgfplots}
\usepackage{newlfont}
\usepackage{tikz}
\usepackage{enumitem}
\usepackage{rotating}
\usepackage{adjustbox}
\usepackage{overpic}
\usepackage{float}

\usepackage[colorlinks=true, linkcolor=blue, urlcolor=blue, citecolor=blue]{hyperref}
\usepackage{amsmath,amsthm,amssymb}
\usepackage{subfigure}

\usepackage{graphicx}
\usepackage{amsmath, amsfonts}
\usepackage{dsfont}
\usepackage{booktabs}
\usepackage{fourier}
\usepackage{xcolor}
\usepackage{array}
\usepackage{makecell}
\usepackage{algorithm}
\usepackage{algpseudocode}
\usepackage{lscape}
\usepackage{threeparttable}
\usepackage{hyperref}
\usepackage[T1]{fontenc}
\usepackage[tableposition=top]{caption}
\usepackage{tabularx}
\newtheorem{thm}{Theorem}[section]

\newtheorem{lem}[thm]{Lemma}

\newtheorem{defn}[thm]{Definition}

\numberwithin{equation}{section}
\usepackage{mathrsfs}

\begin{document}
\title{\textbf{Poisson Image Denoising Using Minimax Concave  and 
		Reweighted $\ell_1$ Penalties: Nonblind  and Blind Approaches}}
\author{
		Reza Parvaz$^a$ \footnote{Corresponding author, Email:
		\href{rparvaz@uma.ac.ir}{rparvaz@uma.ac.ir},~
		\href{mailto:reza.parvaz@yahoo.com}{reza.parvaz@yahoo.com}}
}
\date{}
\maketitle
\begin{center}
$^a$Department of Engineering Sciences, Faculty of Advanced Technologies, 
University of Mohaghegh Ardabili, Ardabil, Iran
\end{center}
\begin{abstract}
\noindent
Images are important tools in various sciences.
Despite the development of photo-taking tools, creating 
clear and image without noise 
remains challenging in practice.
In particular, Poisson noise has an effect  on  medical and astronomical images,
and reduces their quality. Additionally,  blur is another factor that has an effect on image quality.
The problem of image restoration becomes very complicated when we have no information about the Point 
Spread Function (PSF). These types of problems are known as blind case.
However, in some images, such as some astronomical images, the type of PSF
 can be specified, and these types of problems are known as nonblind problems.
Total Variation (TV) is a widely used method for solving such inverse problems, 
where the selection of the penalty function is the most critical factor that affects
the method's performance. 
In this paper, to improve edge preservation, we employ a reweighted $\ell_1$-regularization of the 
fractional order derivative. Furthermore, we propose a nonblind and blind image deblurring approach 
under Poisson noise using the Minimax Concave Penalty (MCP), which is a continuous, sparsity promoting, 
and nearly unbiased regularizer. This formulation leads to a nonconvex optimization model.
To solve the proposed model, we introduce an efficient numerical algorithm based on the Alternating 
Direction Method of Multipliers (ADMM) and provide an analysis of its convergence. Finally, the 
effectiveness of the proposed algorithm are demonstrated through extensive experiments on various images.
\end{abstract}
\vskip.3cm

\section{Introduction}
\subsection{Problem Description}
From space exploration in astronomy to cellular analysis in medicine, images are  
an important tool in scientific discovery. Consequently, obtaining clear images free from 
noise and blur is vital for this goal. However, images taken by photographic tools may be degraded by both noise and blur. 
Among various types of noise, Poisson noise is commonly observed in medical and astronomical imaging data.
In this study, let $X, Y \in \mathbb{R}^{n \times m}$ represent the latent sharp image and 
the observed degraded image, respectively. The image blurring and noise processes are modeled as:
\[ Y = \mathrm{P}(K \circledast X), \]
where $\circledast$ denotes the 2D convolution operator, and $K$ represents the point 
spread function (PSF). Furthermore, the operator $\mathrm{P}(\cdot)$ represents the 
Poisson noise corruption process applied to the blurred image.
Although this equation is commonly expressed in vector form as follows:
\[ y = \mathrm{P}(k\,x), \]
where $x$ and $y$ are the vectorized forms of $X$ and $Y$, respectively, and $k$ is 
the matrix representation of the convolution operator associated with the PSF 
and the imposed boundary conditions.
 In image restoration, to reduce the computational, the periodic condition boundary 
 conditions are usually chosen. With this choice, this matrix becomes a block circulant with circulant blocks (BCCB) and can be constructed using the Fast Fourier transform
 (FFT) \cite{1a}.
Using the Poisson distribution for the model, we have:
\begin{align*}
\mathrm{P}(y|kx)=\prod_{i} \frac{(kx)^{y_1}_i e^{(kx)_i}}{y!}.	
\end{align*}	
Based on the negative log-likelihood of the Poisson distribution, the following 
variational model is commonly employed to recover the latent sharp image:
\begin{align*}
\min_{x \geq 0} \mu \langle \mathbf{1}, kx-y\log(kx) \rangle+\phi(x),
\end{align*}	
where $\phi(x)$ is the regularization function, and $\mu>0$ is the regularization 
parameter that balances the data-fidelity and regularization terms.\\

\subsection{Literature Review and Motivation}
In recent years, many works have studied various variational models for 
image restoration under Poisson noise. A review of these studies shows 
that the choice of the regularization function has a significant impact on 
both the quality of the restored image and the performance of the proposed
 model.
One of the most critical components of an image is its edges; consequently, 
restoring edges in blurred and noisy images is an important part of image restoration. 
To achieve 
edge preserving image restoration while mitigating 
staircase artifacts, researchers have used derivative based methods, 
see for example, \cite{1,2,3,4}. Furthermore, 
high order derivatives have been successfully applied in \cite{5,6}.
Although higher order derivatives have been widely used in image 
restoration and edge detection, they suffer from several limitations 
\cite{7,8}. For example, Laplacian-based regularization, which used 
on second-order derivatives, is isotropic and therefore tends to enhance 
point-like features rather than line-like features, limiting its ability 
to preserve elongated structures and edges \cite{8}. 
To overcome these weaknesses, recent studies have turned to 
fractional order derivatives, which show a better balance between edge 
preservation and smoothness \cite{9,10,11}. 
Another effective tool in regularization based methods is the use of norms.
The $\ell_0$ norm has been widely used as an effective tool in 
image restoration because of its ability to promote sparse 
solutions. It has been extensively studied in the literature, 
as demonstrated in works such as \cite{12,13,14}. 
However, despite its advantages, this 
norm suffers from several significant limitations. Specifically, it is non-convex, 
discontinuous, and non-differentiable, 
and the associated optimization problem is generally NP-hard. These challenges make it 
computationally difficult to solve in practice. Then, this norm is often replaced by 
alternative regularization methods, as the $\ell_1$ norm, which provides a convex approximation 
and enables efficient optimization while still promoting sparsity.
One of the tools that has been proposed to improve the 
$\ell_1$
norm is the use of the reweighted 
$\ell_1$
norm, which has been shown to outperform standard 
$\ell_1$
minimization by requiring substantially 
fewer measurements for exact sparse recovery
\cite{15,16}.
Another approach that has been employed to enhance the performance of 
standard norms is the Minimax Concave Penalty (MCP) \cite{17}. Although this penalty 
is non-convex and consequently renders the optimization problem non-convex, 
it remains attractive because non-convex penalties generally 
offer a closer approximation to the 
$\ell_0$
norm. Therefore, despite the computational difficulties introduced by 
non-convexity, the use of such penalties is justified by 
their superior performance in promoting sparsity.
This non-convex penalty has been widely used in various papers for image restoration, for example
see
\cite{18,19,20,21}.
The use of total variation (TV) for image denoising models leads to the creation of 
staircase artifacts \cite{22}. One of the tools that has been proposed to 
address this issue is the use of total generalized variation (TGV) regularization \cite{23,24}.
 In \cite{18}, an MCP-based TGV model is used for Poisson image denoising. 
 The use of the MCP penalty leads to improved sparsity promotion and better 
 preservation of image edges and fine details.
 Motivated by these advantages, we also employ the MCP regularization in 
 our proposed model. However, instead of using the TGV framework, we combine 
 the MCP regularization in a transform domain with an adaptive reweighted 
 fractional order total variation regularizer.
 The transform-domain MCP term effectively promotes sparse representations of 
 the image, whereas the reweighted fractional-order regularization exploits the 
 nonlocal characteristics of fractional derivatives and adaptively enhances gradient sparsity.
 Compared with the MCP-TGV model, the proposed model provides an alternative 
 regularization framework that combines transform-domain nonconvex sparsity with 
 adaptive fractional order regularization.\\

The remainder of this paper is organized as follows. Section \ref{sec2} presents 
the preliminary concepts and mathematical background required throughout the paper. 
In Section \ref{sec3}, we introduce the proposed variational model and develop an 
efficient numerical algorithm for solving the resulting optimization problem.
We also investigate the convergence of the proposed algorithm in this section.
Section \ref{sec4} presents several numerical experiments to evaluate the effectiveness 
of the proposed method and demonstrate its performance in Poisson image restoration.

\section{Preliminaries and Notation} \label{sec2}
In this section, we introduce the notation, definitions, and mathematical tools used 
throughout this paper. Let $\mathbf{F}$ and $\mathbf{F}^{-1}$ denote the Fast Fourier 
Transform (FFT) and its inverse (IFFT), respectively. 
The symbol $\Gamma(\cdot)$ denotes the gamma function.
Let $W$ be the framelet transform matrix constructed from the low-pass and high-pass filters
 $h_0=\frac{1}{4}[1,2,1], h_1=\frac{\sqrt{2}}{4}[1,0,-1]$and 
$h_2=\frac{1}{4}[-1,2,-1]$. The symbol $\odot$ denotes the Hadamard product.
\\

\begin{defn}
(Fractional-order gradient operator \cite{25}) 	
Let $\Omega \subseteq \mathbb{R}^{2}$ be an open set, and let $f:\Omega\rightarrow\mathbb{R}$ be a real-valued function. The fractional-order Grunwald--Letnikov (G--L) gradient operator is defined by
\begin{align*}
	\nabla^\beta f
=	
	\big(\nabla_h^\beta f,\nabla_v^\beta f\big)^T,
	\qquad
	\beta\in\mathbb{R}^{+},
\end{align*}
where $\beta$ denotes the order of the fractional derivative. The horizontal and 
vertical fractional derivatives are given by

\begin{align*}
	&\nabla_h^\beta f(i,j)
	=
	\sum_{l=0}^{L-1}
	(-1)^l
	\frac{\Gamma(\beta+1)}
	{\Gamma(l+1)\Gamma(\beta-l+1)}
	f(i-l,j),\\
	&\nabla_v^\beta f(i,j)
	=
	\sum_{l=0}^{L-1}
	(-1)^l
	\frac{\Gamma(\beta+1)}
	{\Gamma(l+1)\Gamma(\beta-l+1)}
	f(i,j-l).
\end{align*}
\end{defn}

\begin{defn}
	(Minimax Concave Penalty (MCP)\cite{17})
	The MCP function is defined by
	\begin{align*}
		h_{\gamma,\eta}(x)
		= \int_{0}^{|x|} \left(  \gamma- \frac{|t|}{\eta} \right)_{+}\, dt,
	\end{align*}
	where \(\gamma > 0\) and \(\eta > 1\) are parameters.
\end{defn}

\begin{defn}(Vector MCP \cite{25a})
	For a vector \(x \in \mathbb{R}^n\), the MCP (minimax concave penalty) is defined as
	\[
	\|x\|_{MCP} = \sum_{i=1}^{n} h_{ \gamma,\eta}(x_i),
	\]
	where \(\gamma\) and \(\eta \) are parameters of the MCP function.
\end{defn}

Similar to \cite{25a}, we obtain the following lemma.
\begin{lem}\label{lem1}
	Let $y \in \mathbb{R}^n$, and let $\gamma > 0$ and $\eta > 1$ be parameters of the 
	MCP function. Consider the optimization problem
	\[
	x^* = \arg\min_x \left\{ \frac{1}{2}\|x - y\|_2^2 + \alpha \|x\|_{MCP} \right\},
	\]
	where $\alpha>0$.
	Then the solution is given by
	\[
	x^* =
	\operatorname{sign}(y)\odot\,
	\min\left\{|y|,\max\left\{\frac{\gamma}{\gamma-\alpha}\big(|y|-\alpha \eta\big ),0\right\}\right\}.
	\]
\end{lem}
\section{Proposed Model and Numerical Algorithm}\label{sec3}
In this section, we first present the proposed image restoration model. Next, 
we introduce an algorithm for its numerical solution, and finally, we analyze the 
convergence of the proposed algorithm.
\subsection{Proposed Model}
The following model is proposed for the deconvolution of Poissonian blurred images:
\begin{align}\label{mod1}
\min_{x} \mu \langle \mathbf{1}, kx-y\log(kx) \rangle+\lambda \|Wx\|_{MCP}+\|\nabla^{\beta}x\|_{\omega,1}+\delta_{\mathbb{R}^{+}}(x), 
\end{align}
where $\delta_{\mathbb{R}^{+}}$ denotes the indicator function of the 
set $\mathbb{R}^{+}$ and is added to enforce $x\geq 0$; also, $\|\cdot\|_{\omega,1}$ 
denotes the Reweighted $\ell_{1}$ norm.
In this model, the combination of the MCP penalty and the reweighted 
$\ell_1$ norm plays a key role in enforcing sparsity 
in the transform domain while simultaneously protecting 
 edge structures.

\subsection{Numerical Algorithm}
To solve problem \eqref{mod1}, we employ the alternating direction
method of multipliers (ADMM). We introduce the auxiliary
variables $v$, $g$, $z$ and $m$, and reformulate problem
\eqref{mod1} as
\begin{align*}
	\min_{x,v,g,z,m}\;&
	\mu \langle \mathbf{1}, v-y\log(v) \rangle
	+\lambda \|g\|_{MCP}
	+\|z\|_{\omega,1}
	+\delta_{\mathbb{R}^{+}}(m),
	\\
	\text{s.t.}\;&
	v = kx,\quad
	g = Wx,\quad
	z = \nabla^{\beta}x,\quad
	m = x.
\end{align*}
The augmented Lagrangian associated for the above problem is given by
\begin{align*}
	\mathrm{L}_{\gamma}(x,v,g,z,m;\mathbf{p})
	=&\,
	\mu \left\langle \mathbf{1},
	v-y\log(v)\right\rangle
	+\lambda \|g\|_{MCP}
	+\|z\|_{\omega,1}
	+\delta_{\mathbb{R}^{+}}(m)
	\nonumber\\
	&+\langle p_1,kx-v \rangle
	+\frac{\gamma_1}{2}\|kx-v\|_2^2
	\nonumber\\
	&+\langle p_2,Wx-g \rangle
	+\frac{\gamma_2}{2}\|Wx-g\|_2^2
	\nonumber\\
	&+\langle p_3,\nabla^{\beta}x-z \rangle
	+\frac{\gamma_3}{2}\|\nabla^{\beta}x-z\|_2^2
	\nonumber\\
	&+\langle p_4,x-m \rangle
	+\frac{\gamma_4}{2}\|x-m\|_2^2,
\end{align*}
where
$
\mathbf{p}:=(p_1,p_2,p_3,p_4)
$
denotes the vector of Lagrange multipliers corresponding to the
constraints. Now, by separating the problem into several 
subproblems, we find the solution to each.\\

\noindent For the $v$-subproblem, we consider
\begin{align*}
	\min_{v}\;
	\mu \left\langle\mathbf{1}, v-y\log(v)\right\rangle
	+\frac{\gamma_1}{2}
	\left\|kx-v+\gamma_1^{-1}p_1\right\|_2^2.
\end{align*}
By setting the first-order optimality condition to zero and solving
the resulting quadratic equation componentwise, we obtain the
closed form solution
\begin{align}\label{Upv}
	v=
	\frac{
		\gamma_1 kx+p_1-\mu
		+
		\sqrt{
			\left(\mu-\gamma_1 kx-p_1\right)^2
			+4\mu\gamma_1 y
		}
	}{2\gamma_1}.
\end{align}
For the $x$-subproblem, we consider
\begin{align*}
	\min_{x}\;
	&\frac{\gamma_1}{2}
	\left\|kx-v+\gamma_1^{-1}p_1\right\|_2^2
	+\frac{\gamma_2}{2}
	\left\|Wx-g+\gamma_2^{-1}p_2\right\|_2^2
	\nonumber\\
	&+\frac{\gamma_3}{2}
	\left\|\nabla^{\beta}x-z+\gamma_3^{-1}p_3\right\|_2^2
	+\frac{\gamma_4}{2}
	\left\|x-m+\gamma_4^{-1}p_4\right\|_2^2.
\end{align*}
Assuming periodic boundary conditions, the operators
$k^{T}k$ and $(\nabla^{\beta})^{T}\nabla^{\beta}$ are BCCB matrices
and can therefore be diagonalized by the discrete Fourier transform.
Consequently, the solution can be computed efficiently in the Fourier
domain as
\begin{align}\label{Upx}
	x
	=
	\mathbf{F}^{-1}
	\Bigg(
	\frac{
		\gamma_1 \overline{\mathbf{F}(k)}\,\mathbf{F}(\xi_1)
		+\gamma_2 \mathbf{F}(W^{T}\xi_2)
		+\gamma_3 \overline{\mathbf{F}(\nabla^{\beta})}\,\mathbf{F}(\xi_3)
		+\gamma_4 \mathbf{F}(\xi_4)
	}{
		\gamma_1 |\mathbf{F}(k)|^{2}
		+\gamma_3 |\mathbf{F}(\nabla^{\beta})|^{2}
		+\gamma_2+\gamma_4
	}
	\Bigg).
\end{align}
where
\[
\xi_1:=v-\gamma_1^{-1}p_1,\qquad
\xi_2:=g-\gamma_2^{-1}p_2,
\]
\[
\xi_3:=z-\gamma_3^{-1}p_3,\qquad
\xi_4:=m-\gamma_4^{-1}p_4.
\]
For the $z$-subproblem, we consider
\begin{align*}
	\min_{z}\;
	\|z\|_{\omega,1}
	+\frac{\gamma_3}{2}
	\left\|\nabla^{\beta}x - z + \gamma_3^{-1}p_3\right\|_2^2.
\end{align*}

The solution is given by the soft-thresholding operator:
\begin{align}\label{Upz}
	z
	=
	\operatorname{sign}\!\left(\nabla^{\beta}x+\gamma_3^{-1}p_3\right)
	\odot
	\max\!\left(
	\left|\nabla^{\beta}x+\gamma_3^{-1}p_3\right|
	-\frac{\omega}{\gamma_3},
	\,0
	\right).
\end{align}
Also, following \cite{25b}, the weights associated with the $\ell_1$ norm 
are updated at each iteration based on the magnitude of the image gradient:
\begin{align*}
	\omega_i=\frac{1}{|\nabla^{\beta} x_i|+\epsilon},
\end{align*}
where $\epsilon>0$ is a small constant used to prevent division by zero.

For sub problem $g$, we consider
\begin{align}\label{mce}
\min_{g} \frac{\lambda}{\gamma_2}\|g\|_{\text{MCP}}+\frac{1}{2}
\|g-(Wx+\gamma^{-1}_2p_2)\|^2_2.
\end{align}
To solve problem \eqref{mce}, we employ the Lemma \ref{lem1}, 
 which leads to the following solution.
\begin{align}\label{upg}
	g=\operatorname{sign}(y)\odot\,
	\min\left\{|a|,\max\left\{\frac{\gamma}{\gamma-\frac{\lambda}{\gamma_2}}(|a|-\frac{\lambda\,\eta}{\gamma_2} ),0\right\}\right\},
\end{align}
where $a:=Wx+\gamma^{-1}_2p_2$.

The sub problem with respect to $m$ is given by
\begin{align*}
	\min_{m}\;
	\frac{\gamma_4}{2}\left\|x-m+\gamma_4^{-1}p_4\right\|_2^2
	+\delta_{\mathbb{R}^+}(m).
\end{align*}
Using the definition of the indicator function, the solution is obtained
as the Euclidean projection of
$x+\gamma_4^{-1}p_4$ onto the nonnegative orthant $\mathbb{R}^+$.
Hence, the solution is given by
\begin{align}\label{Upm}
	m
	=
	\max\!\left(x+\gamma_4^{-1}p_4,\;0\right),
\end{align}
where the maximum is taken componentwise.

After solving the above subproblems at the $j$th iteration,
the Lagrange multipliers $p_i$\,($i=1,\ldots,4$), are updated as follows:
\begin{align}\label{Upp1}
	p^{(j+1)}_1
	&=
	p^{(j)}_1+\gamma_1\bigl(kx-v\bigr),\\
	p^{(j+1)}_2
	&=
	p^{(j)}_2+\gamma_2\bigl(Wx-g\bigr),\\
	p^{(j+1)}_3
	&=
	p^{(j)}_3+\gamma_3\bigl(\nabla^{\beta}x-z\bigr),\\ \label{Upp4}
	p^{(j+1)}_4
	&=
	p^{(j)}_4+\gamma_4\bigl(x-m\bigr).
\end{align}

Additionally, $\gamma_i\,(i=1,\cdots,4)$ are updated at each iteration as follows:
\begin{align}\label{Upgam}
	\gamma^{(j+1)}_i=\sigma_i \gamma^{(j)}_i,
\end{align}
where $\sigma_i>1\,(i=1,\cdots,4)$ are fixed constants.

\begin{algorithm}[ht!]
	\caption{ADMM algorithm for solving problem \eqref{mod1}}
	\label{alg:admm_solver}
	\begin{algorithmic}[1]
		\State \textbf{Input:} Blurred and noisy image $y$ and blur kernel $k$
		\State \textbf{Parameters:} Set $ \alpha, \beta, \lambda, \mu, \{\gamma_i\}^{4}_{i=1}$
		\State \textbf{Initialize:} $x^{(0)}=y, g^{(0)}=W y, z^{(0)}=\nabla^{\beta}y, m^{(0)}=0, \mathbf{p}^{(0)}=0, j=0$
		\State $j \gets 0$
		\Repeat
		\State Update $v^{(j+1)}$ by Eq.~\eqref{Upv}
		\State Update $x^{(j+1)}$ by Eq.~\eqref{Upx}
		\State Update $g^{(j+1)}$ by Eq.~\eqref{upg}
		\State Update $z^{(j+1)}$ by Eq.~\eqref{Upz}
		\State Update $m^{(j+1)}$ by Eq.~\eqref{Upm}
		\State Update $\mathbf{p}^{(j+1)}$ by Eqs.~\eqref{Upp1}--\eqref{Upp4}
		\State Update $\gamma_i$ by Eqs.~\eqref{Upgam}
		\State $j \gets j+1$
		\Until{a stopping criterion is satisfied}
		
		\State \textbf{Output:} Deblurred and denoised image $x$
	\end{algorithmic}
\end{algorithm}

\subsection{Blind Problem}
Another important class of image deblurring problems is the blind type. 
Unlike in the previous section, where the PSF is assumed known, blind deblurring 
assumes no prior information about the PSF, and the goal is to estimate both the PSF and the clear image.
Before presenting the method for blind case, we review the expectation-maximization (EM) method for solving the following problem.
More details on this can be seen in \cite{25c,29}. Consider:
\begin{align*}
\min_{x \geq 0} \mu \langle \mathbf{1}, kx-y\log(kx) \rangle,
\end{align*}
by using the Karush-Kuhn-Tucker (KKT) conditions there is an Lagrange multiplier vector
as $\lambda_l$ such that:
\begin{align*}
& k^{\ast}\mathbf{1}-k^{\ast} \Big( \frac{y}{k x} \Big)-\lambda_l,\\
& \lambda_l\odot x=0,
\end{align*}
where $k^{\ast}$ denotes the adjoint operator of $k$.
In the remainder of this section, as in most articles, we omit the Hadamard product symbol for simplicity.
By using above relations and fixed-point iteration, we obtain the following numerical method:
\begin{align}\label{EM1}
x^{(j+1)}=\text{EM}(y,k,x^{(j)}):=\frac{x^{(j)}}{k^{\ast}\mathbf{1}}
k^{\ast}\Big(\frac{y}{k u^{(j)}}\Big).
\end{align} 
Now let's start the proposed algorithm. For the blind problem, we consider the following model:
\begin{align*}
	\min_{x \in \Omega, k \in \Upsilon} \mu \langle \mathbf{1}, kx-y\log(kx) \rangle+\lambda \|Wx\|_{\text{MCP}}
	+\|\nabla^{\beta}x\|_{\omega,1}, 
\end{align*}
where $\Omega$ is image domain and $\Upsilon=\{k\in \mathbb{R}^{s \times q}; k_{i,j}\geq 0 \& \sum_{i,j}k_{i,j}=1\}$.
In the first step by using the optimal condition and fixed-point iteration for above problem, we obtain:
\begin{align}\label{EM2}
	\frac{1}{\mu}\big(\lambda P^{{(j+1)}}_{MCP}+P^{(j+1)}_{\omega,1} \big) +k^{\ast}\mathbf{1}-k^{\ast}\Big( \frac{y}{kx^{(j)}}\Big)=0,
\end{align}	
where $P^{(j+1)}_{\text{MCP}} \in \partial \|W x^{(j+1)} \|_{MCP}, P^{(j+1)}_{\omega,1} \in  \partial \|\nabla^{\beta} x^{(j+1)} \|_{\omega,1}$.
By inserting $\frac{x^{(j)}}{x^{(j+1)}}$ into the last sentence in \eqref{EM2} and using relation \eqref{EM1}, the following relation is obtained:
\begin{align*}
	\frac{1}{\mu k^{\ast}\mathbf{1}}\big(\lambda P^{(j+1)}_{MCP}+P^{(j+1)}_{\omega,1} \big) +\mathbf{1}- \frac{x^{(j+\frac{1}{2})}}{x^{(j+1)}}=0,
\end{align*}	
where 
\begin{align} \label{mblind1}
x^{(j+\frac{1}{2})}:=\text{EM}(y,k,x^{(j)}).
\end{align}
It can be easily seen that the $x^{(j+1)}$ is the solution of the following problem:
\begin{align*}
	\min_{x, k \in \omega} \mu \langle k^{\ast}\mathbf{1}, x-x^{(j+\frac{1}{2})}\log(x) \rangle+\lambda \|Wx\|_{\text{MCP}}
	+\|\nabla^{\beta}x\|_{\omega,1}.
\end{align*}
The method of solving this model is exactly the same as the proposed nonblind model. In the following, to avoid excessive use of symbols, we try to repeat the same symbols.
By introducing the auxiliary variables $\bar{v}$, $\bar{g}$, $\bar{z}$ and $\bar{m}$, we reformulate the model as:
\begin{align*}
&\min_{x, k \in \omega} \mu \langle k^{\ast}\mathbf{1}, x-x^{(j+\frac{1}{2})}\log(x) \rangle+\lambda \|\bar{g}\|_{\text{MCP}}
	+\|\bar{z}\|_{\omega,1},\\
& \text{s.t.} \quad	\bar{g}=W\bar{m},\quad \bar{z}=\nabla^{\beta}\bar{m}, \quad \bar{m}=x.
\end{align*}
The augmented Lagrangian associated for  the above problem can be written as
\begin{align*}
	\mathrm{L}_{\gamma}(x,\bar{g},\bar{z},\bar{m};\bar{\mathbf{p}})
	=&\mu \langle k^{\ast}\mathbf{1}, x-x^{(j+\frac{1}{2})}\log(x) \rangle
	+\lambda \|\bar{g}\|_{\text{MCP}}
	+\|\bar{z}\|_{\omega,1}\\
	&+\langle \bar{p}_1, W\bar{m}-\bar{g} \rangle+\frac{\gamma_1}{2}\|W\bar{m}-\bar{g}\|^2_2\\
	&+\langle \bar{p}_2,\nabla^{\beta}\bar{m}-\bar{z}\rangle+\frac{\gamma_2}{2} \|W\bar{m}-\bar{g}\|^2_2\\
	& +\langle \bar{p}_3,\bar{m}-x\rangle+\frac{\gamma_3}{2} \|\bar{m}-x\|^2_2.
\end{align*}
Now, by dividing the problem into several subproblems, we calculate the solution for each.
For the $x$ subproblem, we have
\begin{align*}
\min_x \mu \langle k^{\ast}\mathbf{1}, x-x^{(j+\frac{1}{2})}\log(x) \rangle
+\frac{\gamma_3}{2}\| \bar{m}-x+\gamma_3^{-1}\bar{p}_3\|^2_2.
\end{align*}
The closed form solution by using
quadratic equation componentwise, is obtained as
\begin{align}\label{mblind2}
	x=
	\frac{-(\mu \chi -\gamma_3 \bar{m}-\bar{p}_3)
		+
		\sqrt{
			\left(\mu \chi -\gamma_3 \bar{m}-\bar{p}_3\right)^2
			+4\mu\gamma_3 \chi x^{(j+\frac{1}{2})}
		}
	}{2\gamma_3},
\end{align}
where $\chi:=k^{\ast} \mathbf{1}$.
The $\bar{m}$ subproblem is given as:
\begin{align*}
\min_{\bar{m}}\frac{\gamma_1}{2}\|W\bar{m}-\bar{g}+\gamma_2^{-1}\bar{p}_1\|^2_2
+\frac{\gamma_2}{2}\| \nabla^{\beta}\bar{m}-\bar{z}+\gamma^{-1}_3\bar{p}_2\|^2_2
+\frac{\gamma_3}{2}\|\bar{m}-x+\gamma^{-1}_3\bar{p}_3\|^2_2.
\end{align*}
By using the discrete Fourier transform, the closed form solution is obtained as
\begin{align}\label{mblind3}
\bar{m}
	=
	\mathbf{F}^{-1}
	\Bigg(
	\frac{
      \gamma_1 \mathbf{F}(W^{T}\xi_1)
		+\gamma_2 \overline{\mathbf{F}(\nabla^{\beta})}\,\mathbf{F}(\xi_2)
		+\gamma_3 \mathbf{F}(\xi_3)
	}{
		\gamma_2 |\mathbf{F}(\nabla^{\beta})|^{2}
		+\gamma_1+\gamma_3
	}
	\Bigg),
\end{align}
where
\[
\xi_1:=\bar{g}-\gamma_2^{-1}\bar{p}_1,
\qquad
\xi_2:=\bar{z}-\gamma_3^{-1}\bar{p}_2,\qquad
\xi_3:=x-\gamma_4^{-1}\bar{p}_3.
\]

The $\bar{g}$ subproblem is written as
\begin{align*}
\min_{\bar{g}} \lambda \|\bar{g}\|_{\text{MCP}}+\frac{\gamma_1}{2}\|\bar{g}-(W\bar{m}+\gamma^{-1}_2 \bar{p}_1) \|^2_2.
\end{align*}
The close form solution is given as:
\begin{align}\label{mblind5}
	\bar{g}=\operatorname{sign}(y)\odot\,
	\min\left\{|\bar{a}|,\max\left\{\frac{\gamma}{\gamma-\frac{\lambda}{\gamma_2}}\big(|\bar{a}|-\frac{\lambda\,\eta}{\gamma_2} \big),0\right\}\right\},
\end{align}
where $\bar{a}:=Wm+\gamma^{-1}_2p_2$.
The final suvproblem $z$ is considered as
\begin{align*}
\min_{\bar{z}} \|\bar{z}\|_{\omega,1}+\frac{\gamma_2}{2}\| \nabla^{\beta}\bar{m}-\bar{z}+\gamma^{-1}_2\bar{p}_2\|^2_2.
\end{align*}
The closed form solution for the above problem is obtained as 
\begin{align}\label{mblind4}
	\bar{z}
	=
	\operatorname{sign}\!\left(\nabla^{\beta}\bar{m}+\gamma_2^{-1}\bar{p}_2\right)
	\odot
	\max\!\left(
	\left|\nabla^{\beta}\bar{m}+\gamma_2^{-1}\bar{p}_2\right|
	-\frac{\omega}{\gamma_2},
	\,0
	\right).
\end{align}
Also, the weights associated 
are updated at each iteration as
\begin{align*}
	\omega_i=\frac{1}{|\nabla^{\beta} \bar{m}_i|+\epsilon}.
\end{align*}

Also, the Lagrange multipliers $\bar{p}_i$\,($i=1,\ldots,3$) and $\gamma_i\,(i=1,\cdots,3)$ at step $j+1$ are updated as:
\begin{align}\label{mblind6}
&\bar{p}^{(j+1)}_1=\bar{p}^{(j)}_1+\gamma_2 (W\bar{m}-\bar{g}),\\
&\bar{p}^{(j+1)}_2=\bar{p}^{(j)}_2+\gamma_2 (\nabla^{\beta}\bar{m}-\bar{z}),\\
&\bar{p}^{(j+1)}_3=\bar{p}^{(j)}_3+\gamma_2 (\bar{m}-\bar{g}),\\\label{mblind7}
&\gamma^{(j+1)}_i=\sigma_i \gamma^{(j)}_i,\,(i=1,\cdots,3).
\end{align}
After obtain the $x^{(j+1)}$, in the next step to obtain $k$, we consider matrix $A$
that $K \circledast X^{(j+1)}=A \text{vec}(K)$, then by using EM, $K$ is updated as:
\begin{align}\label{mblind8}
	&K^{(j+1)}=EM(y,A,K^{(j+1)}),\\\label{mblind9}
&K^{(j+1)}=\frac{K^{(j+1)}}{\| K^{(j+1)}\|_1}.
\end{align}
Boundary conditions are influential in restoring images near the borders, and without considering the 
appropriate conditions, ringing artifacts may appear near the image borders.
Various methods have been presented to solve this problem, such as the Fourier domain restoration filter combined with an extrapolated image \cite{26y},
 as well as, the
 cropped convolution output without zero-padding of boundaries, which is described in \cite{29} and is used in proposed algorithm. 
The proposed algorithm for blind image restoration is given in 
Algorithm \ref{blind_alg}.

\begin{algorithm}[ht!]
	\caption{Proposed algorithm for blind image restoration}
	\label{blind_alg}
	\begin{algorithmic}[1]
		\State \textbf{Input:} Blurred and noisy image $y$ 
		\State \textbf{Parameters:} Set $ \alpha, \beta, \lambda, \mu, \{\gamma_i\}^{4}_{i=1}, \rho$
		\State \textbf{Initialize:} $x^{(0)}=y, \bar{g}^{(0)}=W y, \bar{z}^{(0)}=
		\nabla^{\beta}y, \bar{m}^{(0)}=0, \bar{\mathbf{p}}^{(0)}=0, j=0, K^{(0)}=\mathbf{1}$
		\State $ j \gets 0 $
		\Repeat
		\State Update $x^{(j+\frac{1}{2})}$ by Eq.~\eqref{mblind1}
		\State Update  $x^{(j+1)}$ by Eq.~\eqref{mblind2}
		\State Update $\bar{m}^{(j+1)}$ by Eq.~\eqref{mblind3}
		\State Update $\bar{z}^{(j+1)}$ by Eq.~\eqref{mblind4}
		\State Update $\bar{g}^{(j+1)}$ by Eq.~\eqref{mblind5}
		\State Update $\bar{p}_i^{(j+1)}$ and $\gamma^{(j)}_i,\,(i=1,\cdots,3)$ by Eqs.~\eqref{mblind6}--\eqref{mblind7}
		\State Update $K^{(j+\frac{1}{2})}$ by Eq.~\eqref{mblind8}
		\State Update $K^{(j+1)}$ by Eq.~\eqref{mblind9}
		\State $j \gets j+1$
		\Until{a stopping criterion is satisfied}
		\State \textbf{Output:} Deblurred and denoised image $x$ and PSF matrix $k$
	\end{algorithmic}
\end{algorithm}

\section{Convergence Analysis}
The ADMM method is developed for convex problems, and convergence 
is guaranteed for such problems. However, this method also has abundant applications 
in the nonconvex case. Various methods and articles can be used to examine the 
convergence of ADMM for nonconvex problems \cite{26y1,26y2}. In the following, we study the convergence 
of this method following the approach inspired by \cite{18}.
\begin{lem}\label{lm2}
	Suppose that the weights $\omega_i$ are uniformly bounded and the sequence $\{v^j\}$ is  uniformly bounded away from zero.
	 Then the sequence of multipliers $\{p_i^{(j)}\},~i=1,\cdots,4$ and the augmented Lagrangian sequence $\{L_{\gamma^{(j)}}^{(j)}\}$ are bounded.
\end{lem}
\begin{proof}
From the first-order optimality condition for the $v$-subproblem, we obtain component-wise:
\[
\mu \left(1 - \frac{y_i}{v_i^{j+1}}\right) - (p_1^j)_i - \gamma_1 \left( (Kx^{j+1})_i - v_i^{j+1} \right) = 0, \quad \forall \, i.
\]
Then, utilizing the multiplier update rule \eqref{Upp1}, we obtain:
\[
(p_1^{(j+1)})_i = \mu \left(1 - \frac{y_i}{v_i^{(j+1)}}\right), \quad \forall \, i.
\]
By the assumptions that $y_i \geq 0$ and that the sequence $\{v^j\}$ is uniformly bounded away from zero, it immediately follows that $(p_1^{(j+1)})_i \leq \mu$ for all $i$.
Furthermore, since the sequence $\{v^j\}$ is strictly bounded away from zero, there exists a uniform positive constant $\delta > 0$ such that $v_i^{j+1} \geq \delta$ for all $i$. Consequently, we can establish a uniform lower bound:
\[
(p_1^{(j+1)})_i \geq \mu \left(1 - \frac{\max_i y_i}{\delta}\right) > -\infty.
\]
As a result, by defining $M_1= \max \left\{ \mu, \mu \left| 1 - \frac{\max_i y_i}{\delta} \right| \right\}$,
we arrive 
$\|p_1^{(j+1)}\|_{\infty} \leq M_1$.	
For the $g$-subproblem, the first-order optimality condition yields:
\begin{align*}
0&\in \frac{\lambda}{\gamma_2}\partial \|g^{(j+1)}\|_{\text{MCP}}+\big(g^{(j+1)}-(Wx^{(j+1)}+\gamma_2^{-1}p^{(j)}_2)\big)\\
&=\frac{\lambda}{\gamma_2}\partial \|g^{(j+1)}\|_{\text{MCP}}-\frac{1}{\gamma_2}p^{(j+1)}_2,
\end{align*}
then $p^{(j+1)}_2\in \lambda \partial \|g^{(j+1)}\|_{\text{MCP}}$. By the definition of the MCP function, its subgradient is bounded. Therefore, we obtain $\|p^{(j+1)}_2\|_{\infty}\leq \lambda \gamma$.
Similarly, for $z$-subproblem we have:
\begin{align*}
	0\in \partial \|z^{(j+1)}\|_{\omega,1}-p^{(j)}_3
-\gamma_3
(\nabla^{\beta}x^{(j+1)} - z^{(j+1)} ).
\end{align*}
This implies that $p^{(j+1)}_3 \in \partial \|z^{(j+1)}\|_{\omega,1}$. Since 
\[
\partial \|z^{(j+1)}\|_{\omega,1}=\sum_{i \in I_1} \omega_i\, \operatorname{sign}(z^{(j+1)}_i)\,\mathbf{e}_i
+\sum_{i \in I_2}\omega_i\,[-\mathbf{e}_i,\mathbf{e}_i],
\]
where $I_1=\big\{i; z^{(j+1)}_i \neq\,0  \big\}$, $I_2=\big\{i; z^{(j+1)}_i =\,0  \big\}$
and $\mathbf{e}_i$ denotes the $i$-th standard basis vector.
Then we can say 
$\|p^{(j+1)}_3\|_{\infty}\leq \|\omega\|_{\infty}$.\\
To establish the uniform boundedness of the multiplier sequence $\{p_4^{(j)}\}$, 
we combine the first-order optimality condition of the $x$-subproblem with 
the dual update steps, which yields:
\[
K^T p_1^{(j+1)} + W^T p_2^{(j+1)} + (\nabla^\beta)^T p_3^{(j+1)} + p_4^{(j+1)} = 0.
\]
By taking the infinity norm on both sides, and applying the triangle inequality, we obtain:
\[
\|p_4^{(j+1)}\|_\infty \le \|K^T\|_\infty \|p_1^{(j+1)}\|_\infty + \|W^T\|_\infty \|p_2^{(j+1)}\|_\infty + \|(\nabla^\beta)^T\|_\infty \|p_3^{(j+1)}\|_\infty.
\]
Since $K^T$, $W^T$, and $(\nabla^\beta)^T$ are continuous linear operators acting on finite-dimensional spaces, their induced matrix norms are bounded by positive constants, i.e., $\|K^T\|_\infty \le c_1$, $\|W^T\|_\infty \le c_2$, and $\|(\nabla^\beta)^T\|_\infty \le c_3$. Consequently, incorporating these operator bounds along with the previously established uniform bounds for the remaining multipliers, we arrive at:
\[
\|p_4^{(j+1)}\|_\infty \le c_1 M_1 + c_2 (\lambda \gamma) + c_3 \|\omega\|_\infty  < \infty.
\]
This confirms that the sequence $\{p_4^{(j)}\}$ is uniformly bounded in the infinity norm.\\
We can write:
\begin{align*}
L_{\gamma^{(j+1)}}\big(x^{(j+1)},&v^{(j+1)},g^{(j+1)},z^{(j+1)},m^{(j+1)},\mathbf{p}^{(j+1)} \big)-
L_{\gamma^{(j)}}\big(x^{(j+1)},v^{(j+1)},g^{(j+1)},\\
&z^{(j+1)},m^{(j+1)},\mathbf{p}^{(j+1)} \big)= 
\sum^{4}_{i=1} \frac{\gamma^{(j+1)}_i-\gamma^{(j)}_i}{2(\gamma^{(j)}_i)^2}\big\|p^{(j+1)}_i-p^{(j)}_i \big\|^2_2,
\end{align*}
and
\begin{align*}
	L_{\gamma^{(j)}}\big(x^{(j+1)},&v^{(j+1)},g^{(j+1)},z^{(j+1)},m^{(j+1)},\mathbf{p}^{(j+1)} \big)-
	L_{\gamma^{(j)}}\big(x^{(j+1)},v^{(j+1)},g^{(j+1)},\\
	&z^{(j+1)},m^{(j+1)},\mathbf{p}^{(j)} \big)= 
	\sum^{4}_{i=1} \frac{1}{\gamma^{(j)}_i}\big\| p^{(j+1)}_i-p^{(j)}_i \big\|^2_2.
\end{align*}

Considering the subproblems for the values of 
$m,z,g$ and $v$ the following relations can be written:

\begin{align*}
	L_{\gamma^{(j)}}\big(x^{(j+1)},&v^{(j+1)},g^{(j+1)},z^{(j+1)},m^{(j+1)},\mathbf{p}^{(j)} \big)-
	L_{\gamma^{(j)}}\big(x^{(j+1)},v^{(j+1)},g^{(j+1)},\\
	&z^{(j+1)},m^{(j)},\mathbf{p}^{(j)} \big)\leq 0,
\end{align*}
\begin{align*}
	L_{\gamma^{(j)}}\big(x^{(j+1)},&v^{(j+1)},g^{(j+1)},z^{(j+1)},m^{(j)},\mathbf{p}^{(j)} \big)-
	L_{\gamma^{(j)}}\big(x^{(j+1)},v^{(j+1)},g^{(j+1)},\\
	&z^{(j)},m^{(j)},\mathbf{p}^{(j)} \big)\leq 0,
\end{align*}
\begin{align*}
	L_{\gamma^{(j)}}\big(x^{(j+1)},&v^{(j+1)},g^{(j+1)},z^{(j)},m^{(j)},\mathbf{p}^{(j)} \big)-
	L_{\gamma^{(j)}}\big(x^{(j+1)},v^{(j+1)},g^{(j)},\\
	&z^{(j)},m^{(j)},\mathbf{p}^{(j)} \big)\leq 0,
\end{align*}
\begin{align*}
	L_{\gamma^{(j)}}\big(x^{(j+1)},&v^{(j+1)},g^{(j)},z^{(j)},m^{(j)},\mathbf{p}^{(j)} \big)-
	L_{\gamma^{(j)}}\big(x^{(j+1)},v^{(j)},g^{(j)},\\
	&z^{(j)},m^{(j)},\mathbf{p}^{(j)} \big)\leq 0.
\end{align*}

For the $x$-subproblem, taking the second derivative of the objective function with respect to $x$ yields the Hessian matrix:
\[ \mathbf{H} = \gamma_1 k^T k + \gamma_2 W^T W + \gamma_3 (\nabla^{\beta})^T \nabla^{\beta} + \gamma_4 I.
\]
For all $0\neq x \in \mathbb{R}^n$, we have 
\[ x^T\mathbf{H}x = \gamma_1 \| kx\|^2_2 + \gamma_2 \| Wx\|^2_2 + \gamma_3 \|\nabla^{\beta}x\|^2_2 + \gamma_4 \|x\|^2_2>0.
\]
Then the Hessian matrix $\mathbf{H}$ is strictly positive definite. This guarantees that the $x$-subproblem is strongly convex. Consequently, there exists a positive constant $C > 0$ such that the following descent inequality holds:
\begin{align*}
	L_{\gamma^{(j)}}\big(x^{(j+1)}, v^{(j)}, g^{(j)}, z^{(j)}, m^{(j)}; \mathbf{p}^{(j)}\big) &- L_{\gamma^{(j)}}\big(x^{(j)}, v^{(j)}, g^{(j)}, z^{(j)}, m^{(j)}; \mathbf{p}^{(j)}\big)\\
	&\leq -\frac{C}{2} \|x^{(j+1)} - x^{(j)}\|^2_2.
\end{align*}
Then, by combining the above discussion, we have:
\begin{align*}
	L_{\gamma^{(j+1)}}^{(j+1)} - L_{\gamma^{(j)}}^{(j)} \leq \sum_{i=1}^{4} \frac{\gamma^{(j+1)}_i + \gamma^{(j)}_i}{2\big(\gamma_i^{(j)}\big)^2} \big\|p_i^{(j+1)} - p_i^{(j)}\big\|^2_2 - \frac{C}{2}\big\|x^{(j+1)} - x^{(j)}\big\|^2_2.
\end{align*}
In the first step of the proof, it has been shown that the $p_i$'s are bounded. Hence, one can find values such as $M_i$'s such that
$
\|p^{j+1}_i - p^{j}_i\| \le M_i.
$
Consequently by considering \eqref{Upgam}, we have:

\[
L^{(j+1)}_{\gamma^{(j+1)}} - L^{(0)}_{\gamma^{(0)}} \le \sum_{i=1}^{4} \left[ \frac{(\sigma_i+1)M_i}{2\gamma_i^{(0)}} \left( \frac{1 - (\frac{1}{\sigma_i})^{j+1}}{1 - \frac{1}{\sigma_i}} \right) \right] - \frac{C}{2}\sum_{i=0}^{j}\|x^{(i+1)} - x^{(i)}\|_2^2.
\]
Now when 
$j\rightarrow \infty$ we get:

\begin{align}\label{ewww}
\lim_{j \to \infty}L^{(j+1)}_{\gamma^{(j+1)}} - L^{(0)}_{\gamma^{(0)}} \le \sum_{i=1}^{4}  \frac{\sigma_i(\sigma_i+1)  M_i}{2(\sigma_i-1)\gamma_i^{(0)}}  -\lim_{j \to \infty} \frac{C}{2}\sum_{i=0}^{j}\|x^{(i+1)} - x^{(i)}\|_2^2.
\end{align}
then
\begin{align}\label{nneq1}
\lim_{j \to \infty}L^{(j+1)}_{\gamma^{(j+1)}}  \le  L^{(0)}_{\gamma^{(0)}} + \sum_{i=1}^{4}  \frac{\sigma_i(\sigma_i+1)  M_i}{2(\sigma_i-1)\gamma_i^{(0)}}< \infty.
\end{align}
By using norm proparties we can write:

\begin{align}
L^{(j+1)}_{\gamma^{(j+1)}} &= \mu \langle \mathbf{1}, v^{(j+1)} - y\log(v^{(j+1)}) \rangle + \lambda \|g^{(j+1)}\|_{\text{MCP}} + \|z^{(j+1)}\|_{\omega,1} + \delta_{\mathbb{R}^{+}}(m^{(j+1)}) \nonumber\\
&+ \frac{\gamma_1^{(j+1)}}{2}\|kx^{(j+1)} - v^{(j+1)} + \frac{1}{\gamma_1^{(j+1)}}p_1^{(j+1)}\|_2^2 -
\frac{\|p_1^{(j+1)}\|_2^2}{2\gamma_1^{(j+1)}} \nonumber\\ 
&+ \frac{\gamma_2^{(j+1)}}{2}\|Wx^{(j+1)} - g^{(j+1)} +
 \frac{1}{\gamma_2^{(j+1)}}p_2^{(j+1)}\|_2^2 - \frac{\|p_2^{(j+1)}\|_2^2}{2\gamma_2^{(j+1)}}\nonumber\\ &+\frac{\gamma_3^{(j+1)}}{2}\|\nabla^{\beta}x^{(j+1)} - z^{(j+1)} + \frac{1}{\gamma_3^{(j+1)}}p_3^{(j+1)}\|_2^2 -\frac{\|p_3^{(j+1)}\|_2^2}{2\gamma_3^{(j+1)}}\nonumber\\ 
 &+ \frac{\gamma_4^{(j+1)}}{2}\|x^{(j+1)} - m^{(j+1)} +\frac{1}{\gamma_4^{(j+1)}}p_4^{(j+1)}\|_2^2 - \frac{\|p_4^{(j+1)}\|_2^2}{2\gamma_4^{(j+1)}} \label{nneq2}.
\end{align}
Then since $\|\cdot\|_2$, 
 MCP norm, weighted $\ell_1$-norm, and indicator constraint functions are bounded below by $0$,
and since
the Poisson component $v - y\log(v)$ achieves its absolute minimum point when $v = y$, we get
\[
\lim_{j \to \infty} L^{(j+1)}_{\gamma^{(j+1)}} \ge \mu \langle \mathbf{1}, y - y\log(y) \rangle - \lim_{j \to \infty} \left( \sum_{i=1}^{4} \frac{\|p_i^{(j+1)}\|_2^2}{2\gamma_i^{(j+1)}} \right)
= \mu \langle \mathbf{1}, y - y\log(y) \rangle > -\infty.
\]
Therefore based on above dicession,  the sequence $\{L_{\gamma^{(j)}}^{(j)}\}$ is bounded.

\begin{align*}
0 \le \frac{C}{2} \sum_{i=0}^{j} \|x^{(i+1)} - x^{(i)}\|_2^2 \le L^{(0)} - \lim_{j \to \infty} L^{(j+1)} + \sum_{m=1}^{4} \frac{M_m \sigma_m}{2\gamma_m^{(0)}(\sigma_m - 1)}.
\end{align*}
\end{proof}

\begin{lem}\label{lm3}
	The sequences $\{x^{(j)}\}, \{v^{(j)}\}, \{g^{(j)}\}, \{z^{(j)}\},$ and $\{m^{(j)}\}$ generated by the algorithm are bounded if the internal MCP parameters $\gamma$, $\eta$  and regularization parameter $\lambda$ satisfies:
	\begin{align}\label{eqqrr1}
	 L^{(0)} + \sum_{m=1}^{4} \frac{M_m \sigma_m}{2\gamma_m^{(0)}(\sigma_m - 1)}  < \frac{1}{2}\lambda\gamma^2 \eta.
	\end{align}
\end{lem}
\begin{proof}
By using \eqref{nneq1}, \eqref{nneq2} and since
$\lim_{j\rightarrow \infty}\frac{\|p_2^{(j+1)}\|_2^2}{2\gamma_2^{(j+1)}}=0$, we get
\begin{align}\label{nneq3}
\lim_{j \rightarrow \infty } \lambda \|g^{(j+1)}\|_{\text{MCP}} \leq \lim_{j \to \infty}L^{(j+1)}_{\gamma^{(j+1)}} < \frac{1}{2}\lambda\gamma^2 \eta.
\end{align}
Suppose the sequence $\{g^{(j)}\}$ is unbounded. This implies that as $j \to \infty$, there exists at least one coordinate component $i$ such that $|g_i^{(j+1)}| \to \infty$. By the definition of the Minimax Concave Penalty function, when $|g_i^{(j+1)}| \ge \gamma\eta$, the scalar penalty reaches its maximum constant threshold,
$
h_{\gamma,\eta}(g_i^{(j+1)}) = \frac{1}{2}\gamma^2\eta.
$
Since the total vector norm is the sum of non-negative component penalties, we can isolate this diverging index:
\[
\lambda \|g^{(j+1)}\|_{\text{MCP}} \geq \lambda h_{\gamma,\eta}(g_i^{(j+1)}) = \frac{1}{2}\lambda \gamma^2 \eta.
\]
Taking the limit on both sides yields $\lambda \lim_{j\to\infty} \|g^{(j+1)}\|_{\text{MCP}} \geq \frac{1}{2}\lambda \gamma^2 \eta$, which directly contradicts the upper boundary condition established in \eqref{nneq3}. 
The boundedness of the sequences 
$\{x^{(j)}\}, \{v^{(j)}\}, \{z^{(j)}\},$ and $\{m^{(j)}\}$ follows from the boundedness of the sequences 
$\{g^{(j)}\}$ and 
$\{p^{(j)}_i\},~i=1,\cdots,4$  along with Eqs. \eqref{Upp1}-\eqref{Upp4}.
\end{proof}

\begin{thm}
If the internal MCP parameters $\gamma$, $\eta$  and regularization parameter $\lambda$ satisfies \eqref{eqqrr1},
then sequence
$\{X^{(j)}\}:=\{x^{(j)}, v^{(j)}, g^{(j)}, z^{(j)}, m^{(j)}, p_1^{(j)}, p_2^{(j)}, p_3^{(j)}, p_4^{(j)}\}$
is bounded and 
\begin{align}
&\lim_{j \to \infty} \|x^{(j+1)} - x^{(j)}\|_2 = 0, \quad \lim_{j \to \infty} \|v^{(j+1)} - v^{(j)}\|_2 = 0, \quad \lim_{j \to \infty} \|g^{(j+1)} - g^{(j)}\|_2 = 0\\
&\lim_{j \to \infty} \|z^{(j+1)} - z^{(j)}\|_2 = 0, \quad \lim_{j \to \infty} \|m^{(j+1)} - m^{(j)}\|_2 = 0.
\end{align}
\end{thm}
\begin{proof}
The boundedness of the sequence 
$\{X^{(j)}\}$
is a direct consequence of Lemmas \ref{lm2} and \ref{lm3}. By using \eqref{ewww}, we get
\begin{align*}
 0 \le \lim_{j \to \infty}\frac{C}{2}\sum_{i=0}^{j}\|x^{(i+1)} - x^{(i)}\|_2^2
 &\le L^{(0)}_{\gamma^{(0)}} -\lim_{j \to \infty}L^{(j+1)}_{\gamma^{(j+1)}} +\sum_{i=1}^{4}  \frac{\sigma_i(\sigma_i+1)  M_i}{2(\sigma_i-1)\gamma_i^{(0)}}\\
 &\le  L^{(0)}_{\gamma^{(0)}}+\sum_{i=1}^{4}  \frac{\sigma_i(\sigma_i+1)  M_i}{2(\sigma_i-1)\gamma_i^{(0)}} <\infty.
\end{align*}
Because the right-hand side is bounded above by a finite constant, the infinite series converges, and thus
\begin{align}\label{rtt0}
\lim_{i \to \infty}\|x^{(i+1)} - x^{(i)}\|_2=0.
\end{align}
Since when $j \to \infty$ then $\gamma^{(j)}_i \to \infty $, therefore we obtain:
\begin{align}
&\|kx^{(j+1)} - v^{(j+1)}\|_2 = \frac{\|p_1^{(j+1)} - p_1^{(j)}\|_2}{\gamma_1^{(j)}} \to 0 \label{rtt1}, \\
&\|Wx^{(j+1)} - g^{(j+1)}\|_2 = \frac{\|p_2^{(j+1)} - p_2^{(j)}\|_2}{\gamma_2^{(j)}} \to 0,\nonumber\\
&\|\nabla^{\beta}x^{(j+1)} - z^{(j+1)}\|_2 = \frac{\|p_3^{(j+1)} - p_3^{(j)}\|_2}{\gamma_3^{(j)}} \to 0,\nonumber\\
&\|x^{(j+1)} - m^{(j+1)}\|_2 = \frac{\|p_4^{(j+1)} - p_4^{(j)}\|_2}{\gamma_4^{(j)}} \to 0. \nonumber
\end{align}
On the other hand, by the triangle inequality, we obtain
\begin{align*}
\| v^{(j+1)}-v^{(j)}\|_2 \leq \|v^{(j+1)}-k x^{(j+1)}\|_2+\|k(x^{(j+1)}-x^{(j)})\|_2
+\|v^{(j+1)}-kx^{(j)}\|_2.
\end{align*}
Therefore, by using the above relation and \eqref{rtt0}-\eqref{rtt1}, when $j \to \infty$, we can conclude that
$\| v^{(j+1)}-v^{(j)}\|_2 \to 0$.  Similarly, the following relations can be obtained:
\begin{align}
\|g^{(j+1)}-g^{(j)}\|_2 \to 0, \quad \|m^{(j+1)}-m^{(j)}\|_2 \to 0, \quad \|z^{(j+1)}-z^{(j)}\|_2 \to 0.
\end{align}
Boundedness of the sequence 
$\{X^{(j)}\}$
implies the existence of a convergent subsequence 
$\{X^{k_j}\}$
 with limit 
 $\{ X^{\ast}\}$. Since $L_\gamma$
 is the lower semicontinuous, we  have:
 \begin{align}\label{ew1}
 \liminf_{j \to \infty}
  L_{\gamma^{(k_j)}}\big(x^{(k_j+1)},v^{(k_j+1)}, g^{(k_j+1)},z^{(k_j)},m^{(k_j)}, \mathbf{p}^{(k_j)})\geq L_{\gamma^{\infty}}(X^{\ast}\big).
 \end{align}
On the other hand, since $g^{(k_j+1)}$ is obtained by minimizing $L_{\gamma^{(k_j)}}$ with respect to $g$, we have:
\begin{align}\label{ew2}
	\limsup_{j \to \infty} L_{\gamma^{(k_j)}}\big(x^{(k_j+1)}, v^{(k_j+1)}, g^{(k_j+1)}, z^{(k_j)}, m^{(k_j)}, \mathbf{p}^{(k_j)}\big) \leq L_{\gamma^{\infty}}(X^{\ast}).
\end{align}
Considering \eqref{ew1} and \eqref{ew2}, we obtain:
\begin{align}\label{ew3}
	\lim_{j \to \infty} L_{\gamma^{(k_j)}}\big(x^{(k_j+1)}, v^{(k_j+1)}, g^{(k_j+1)}, z^{(k_j)}, m^{(k_j)}, \mathbf{p}^{(k_j)}\big) = L_{\gamma^{\infty}}(X^{\ast}).
\end{align}
Consequently, due to the convergence of the individual components and the vanishing of the primal-dual residuals, the continuity of each constituent term in the objective function follows:
\begin{align*}
	&\lim_{j \to \infty} \langle \mathbf{1}, v^{(k_j+1)} - y \log(v^{(k_j+1)}) \rangle = \langle \mathbf{1}, v^{\ast} - y \log(v^{\ast}) \rangle, \quad \lim_{j \to \infty} \|g^{(k_j+1)}\|_{\text{MCP}} = \|g^{\ast}\|_{\text{MCP}}, \\
	&\lim_{j \to \infty} \|z^{(k_j)}\|_{\omega,1} = \|z^{\ast}\|_{\omega,1}, \quad \lim_{j \to \infty} \delta_{\mathbb{R}_+}(m^{(k_j)}) = \delta_{\mathbb{R}_+}(m^{\ast}).
\end{align*}
From the limits established in \eqref{rtt1}, taking the limit directly yields:
\[
\lim_{j \to \infty} \|Kx^{(k_j+1)} - v^{(k_j+1)}\|_2 = 0 \implies Kx^* - v^* = 0.
\]
The following primal feasibility relations are derived in an identical manner:
\[
Wx^* - g^* = 0, \quad \nabla^{\beta}x^* - z^* = 0, \quad x^* - m^* = 0.
\]
Furthermore, by using the first-order optimality conditions of the auxiliary subproblems, the iterates satisfy:
\begin{align}
	& p_1^{(k_j+1)} \in \partial \left( \mu \langle \mathbf{1}, v^{(k_j+1)} - y\log(v^{(k_j+1)}) \rangle \right), \\
	& p_2^{(k_j+1)} \in \lambda\,\partial \|g^{(k_j+1)}\|_{\text{MCP}}, \\
	& p_3^{(k_j+1)} \in \partial \|z^{(k_j)}\|_{\omega,1}, \\
	& p_4^{(k_j+1)} \in \partial \delta_{\mathbb{R}_+}(m^{(k_j)}) = N_{\mathbb{R}_+}(m^{(k_j)}),
\end{align}
where $N_{\mathbb{R}_+}(\cdot)$ denotes the normal cone to the non-negative orthant. Passing to the limit as $j \to \infty$ via the closed-graph property of the subdifferential operators yields:
\begin{align*}
	& p_1^{\ast} \in  \partial \left( \mu \langle \mathbf{1}, v^{\ast} - y\log(v^{\ast}) \rangle \right), \quad p_2^{\ast} \in \lambda\, \partial \|g^{\ast}\|_{\text{MCP}}, \\
	& p_3^{\ast} \in \partial \|z^{\ast}\|_{\omega,1}, \quad p_4^{\ast} \in N_{\mathbb{R}_+}(m^*). 
\end{align*}
Finally, substituting these relations along with the primal feasibility conditions into the limiting primal optimality relation, $K^T p_1^* +\lambda\, W^T p_2^* + (\nabla^\beta)^T p_3^* + p_4^* = 0$, 
results in:
\[
0 \in K^T \partial \left( \mu \langle \mathbf{1}, Kx^* - y\log(Kx^*) \rangle \right) +\lambda\, W^T \partial \|Wx^*\|_{\text{MCP}} + (\nabla^\beta)^T \partial \|\nabla^\beta x^*\|_{\omega,1} + N_{\mathbb{R}_+}(x^*).
\]
Then limit point $x^{\ast}$  satisfies the Karush-Kuhn-Tucker (KKT) conditions of the optimization problem. 
\end{proof}

\section{Simulations results}\label{sec4}
This section presents numerical results in order to 
evaluate the effectiveness of the proposed algorithm.
For the simulation of the numerical results, MATLAB 2014b is used on a 
system running Windows 10 (64-bit) with 
an Intel\textregistered{} Core\texttrademark{} 
i3-5005U CPU @ 2.00 GHz.
The test images are created using the \texttt{poissrnd} function for Poisson noise corruption, 
while \texttt{imfilter} and \texttt{fspecial} are used for filtering and kernel generation, respectively.
 The performance of the proposed method is evaluated  using two widely recognized metrics:
 peak signal-to-noise ratio (PSNR) and mean structural similarity index (MSSIM) 
 \cite{26a,26b,26c}.
 Let $x$ be the reference (ground-truth) image and $y$ the restored  image.
 The PSNR between two images 
$x$ and 
$y$
 is computed as
 \begin{align*}
 	\text{PSNR}(x,y)= 10 \log_{10} \frac{I_{\max}^2}{\text{MSE}(x,y)},
 \end{align*}
 where 
$I_{\max}$
denotes the maximum pixel value in the image and MSE is the mean squared error.
Also, the MSSIM index is calculated as \footnote{The source code for MSSIM is available at:
	\begin{quote}
		\url{https://www.cns.nyu.edu/~lcv/ssim/} (\cite{26c}) 
\end{quote}}
\begin{align*}
\text{MSSIM}(x,y)=\frac{1}{M}\sum_{j=1}^{M} \text{SSIM}(x_j,x_j),
\end{align*} 
where $M$ is the total number of local windows, and $x_j$ and $y_j$ represent the image patches within the $j$-th local window of images $x$ and $y$, respectively.
In the above equation, SSIM is obtained by
\begin{align*}
	\text{SSIM}(x_j,y_j)=\frac{(2\mu_{x_j}\mu_{y_j}+c_1)(2\sigma_{x_jy_j}+c_2)}{(\mu^2_{x_j}+\mu^2_{y_j}+c_1)(\sigma^2_{x_j}+\sigma^2_{y_j}+c_2)},
\end{align*}
where $\mu_{x_j}$ and $\mu_{y_j}$ denote the local means, $\sigma_{x_j}$ and
 $\sigma_{y_j}$ the local standard deviations, $\sigma_{x_jy_j}$ the local 
 covariance, and $c_1$ and $c_2$ are small positive constants.

Also, to compare the behavior of the proposed algorithm, we benchmark it against
 BM3D \cite{27}, OGS \cite{28}, and FOTV \cite{29}
\footnote{The source codes for the competing methods are available at:
\begin{quote}
		\url{http://www.cs.tut.fi/~foi/invansc/} (BM3D \cite{27}) \\
		\url{https://github.com/KSJhon/PoissonDeblur_hybrid} (OGS \cite{28}) \\
		\url{https://github.com/mujib2020/Non-blind-and-Blind-Deconvolution-under-Poisson-noise} (FOTV \cite{29})
\end{quote}}.
Since the proposed method involves several parameters, we evaluated its performance across a range of parameter values. Specifically, we tested
$\alpha \in \{1, 1.2, 1.5, 1.6, 1.8, 1.9\}$, 
$\mu \in \{2, 3, 5, 6, 7, 8, 10, 15, 20\}$, 
and $\lambda \in \{10^{-j} : j = 1, \dots, 4\}$. For the initialization parameters $\gamma^{0}_i$, we set 
$\gamma^{0}_1 \in \{5 \times 10^{-j} : j = 1, \dots, 4\}$ and $\gamma^{0}_i \in \{10^{-j} : j = 1, \dots, 4\}$ for $i = 2, 3, 4$.
Also, we set $\sigma_i = 1.01$ for $i = 1, \dots, 4$ and reported the best results.
 For the stopping criterion of the proposed algorithm, we run the algorithm for up to 400 iterations. 
 At the $i$-th iteration, if the condition Error$=\|u^{(j)}-u^{(j-1)}\|_2/\|u^{(j)}\|_2 \leq \text{tol}$
 is satisfied with 
 tol$=10^{-5}$, we break the algorithm.
 In the simulation, the images in the datasets of the above codes are used, and some images from the USC-SIPI image database
 \footnote{
 	\url{https://sipi.usc.edu/database/}}  are also used. Figure \ref{figEx0} shows the images used for the simulation part of the presented algorithm.
  \\

\begin{figure}[H]
	\centering
	\makebox[0pt]{
		\subfigure[]{\label{fig:gull}\includegraphics[width=0.2\textwidth]{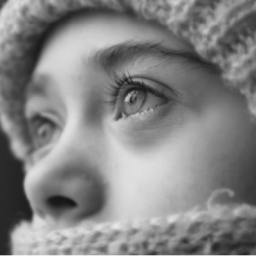}}
		\subfigure[]{\label{fig:tiger}\includegraphics[width=0.2\textwidth]{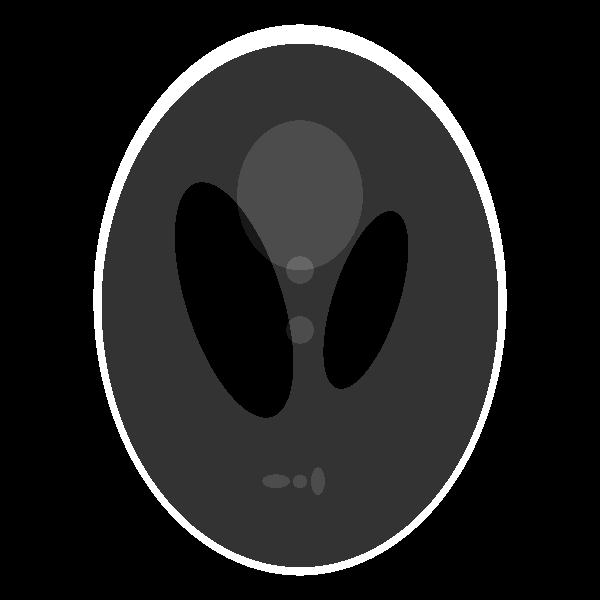}}
		\subfigure[]{\label{fig:mouse}\includegraphics[width=0.2\textwidth]{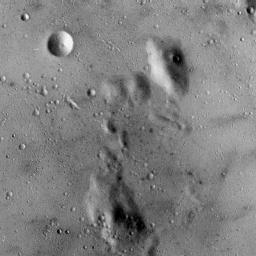}}
		\subfigure[]{\label{fig:gull}\includegraphics[width=0.2\textwidth]{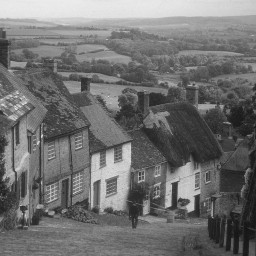}}
		\subfigure[]{\label{fig:tiger}\includegraphics[width=0.2\textwidth]{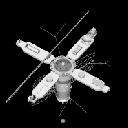}}}
	\vspace{-.2cm}
	\caption{Test images: (a) Girl image, 
		(b) Shepp-Logan Phantom image, (c) \texttt{5.1.09}-Moon surface, 
		(d) Hill image, (e) Satellite}
	\label{figEx0}
\end{figure}

 As a first case study, we chose the Girl ($256 \times 256$), and blurred by Gaussian PSF 
 with size ($7 \times 7$) and standard deviation of 
 $\sqrt{2}$. We then performed deblurring at four different peak intensity levels to evaluate the 
 robustness of the proposed method. The quantitative comparison of the deblurring results, in 
 terms of PSNR and MSSIM, against state-of-the-art methods is summarized in Table \ref{tabEx1}.
 Furthermore, to illustrate a visual comparison of the output results, the reconstructed images 
 for the Girl image are presented and compared in Figure \ref{figEx1}. In this example and the 
 other examples, to demonstrate the accuracy of the proposed method and provide a more detailed 
 comparison, a selected region of the image has been cropped and magnified.\\
 
\begin{center}
	\begin{table}[H]
		\caption{\small{Compared results of	PSNR and MSSIM for different
				peaks of the Girl ($256 \times 256$).}}
		\label{tabEx1}
		\centering
		\small
		\begin{tabular}{ccccccccc}
			\hline
			& \multicolumn{2}{c}{BM3D \cite{27}} &\multicolumn{2}{c}{OGS \cite{28}}  &\multicolumn{2}{c}{FOTV \cite{29}}
			&\multicolumn{2}{c}{Proposed method} \\
			\cmidrule(l){2-3} \cmidrule(l){4-5} \cmidrule(l){6-7}  \cmidrule(l){8-9} 
			Peak  &PSNR &MSSIM &PSNR &MSSIM &PSNR &MSSIM&PSNR &MSSIM\\
			\cmidrule(l){1-1} 		\cmidrule(l){2-3} \cmidrule(l){4-5} \cmidrule(l){6-7}  \cmidrule(l){8-9} 
			255   &31.597 &0.92298 &30.495 &0.92645 &31.352 &0.92548 &\textbf{31.810 }&\textbf{0.92720}\\
			127.5 &30.941 &0.90387 &29.444 &0.90120 &30.532 &0.91346 &\textbf{31.216} &\textbf{0.91416}\\
			51    &29.860 &0.85419 &27.871 &0.84467 &29.437 &0.86550 &\textbf{32.588} &\textbf{0.87879}\\
			25.5  &28.285 &0.77034 &26.567 &0.75651 &28.012 &0.76263 &\textbf{29.266} &\textbf{0.84865}\\
			\hline
		\end{tabular}
	\end{table}
\end{center}

\begin{figure}[H]
	\centering
	\makebox[0pt]{
		\subfigure[]{\label{fig:gull}\includegraphics[width=0.3\textwidth]{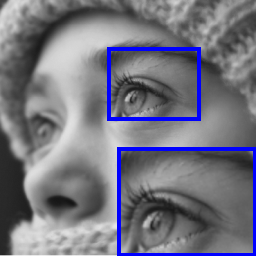}}
		\subfigure[]{\label{fig:tiger}\includegraphics[width=0.3\textwidth]{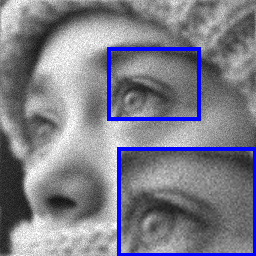}}
		\subfigure[]{\label{fig:mouse}\includegraphics[width=0.3\textwidth]{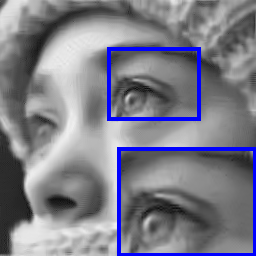}}}
	\\
	\vspace{-.3cm}
	\makebox[0pt]{
		\subfigure[]{\label{fig:gull}\includegraphics[width=0.3\textwidth]{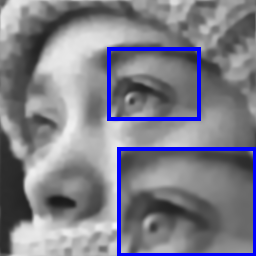}}
		\subfigure[]{\label{fig:tiger}\includegraphics[width=0.3\textwidth]{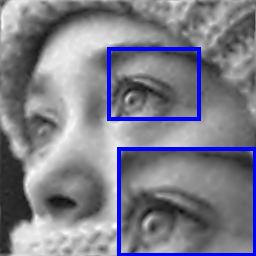}}
		\subfigure[]{\label{fig:mouse}\includegraphics[width=0.3\textwidth]{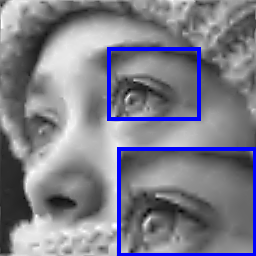}}}
	\vspace{-.2cm}
	\caption{Visual comparison of the Girl images with peak 256: (a) clear input image, 
		(b) blurred and noisy image, (c) BM3D \cite{27}, 
		(d) OGS \cite{28}, (e) FOTV \cite{29}, and (f) proposed method.}
	\label{figEx1}
\end{figure}

In the next example,
 we chose Shepp-Logan Phantom ($600 \times 600$), and blurred by Gaussian PSF 
with size ($9 \times 9$) and standard deviation of 
$\sqrt{3}$.  The PSNR and SSIM results for this image are
 compared in Table \ref{tabEx2}. Additionally, the visual comparison of the output images is presented in Figure \ref{figEx2}.

 \begin{center}
 	\begin{table}[H]
 		\caption{\small{Compared results of	PSNR and MSSIM for different
 				 peaks of Shepp-Logan Phantom ($600 \times 600$).}}
 		\label{tabEx2}
 		\centering
 		\small
 		\begin{tabular}{ccccccccccccccc}
 			\hline
 			& \multicolumn{2}{c}{BM3D \cite{27}} &\multicolumn{2}{c}{OGS \cite{28}}  &\multicolumn{2}{c}{FOTV \cite{29}}
 			&\multicolumn{2}{c}{Proposed method} \\
 			\cmidrule(l){2-3} \cmidrule(l){4-5} \cmidrule(l){6-7}  \cmidrule(l){8-9} 
 			Peak  &PSNR &MSSIM &PSNR &MSSIM &PSNR &MSSIM&PSNR &MSSIM\\
 			 	\cmidrule(l){1-1} 		\cmidrule(l){2-3} \cmidrule(l){4-5} \cmidrule(l){6-7}  \cmidrule(l){8-9} 
 			255   &28.913 &0.93805 &28.393 &0.96687   &32.210 &0.97961 &\textbf{33.852} &\textbf{0.98054}\\
 	    	127.5 &28.857 &0.92046 &28.002 &0.95849   &31.602 &\textbf{0.98288} &\textbf{33.167 }&0.98257\\
 			51    &28.744 &0.9246  &27.252 &0.9365    &30.791 &\textbf{0.97705} &\textbf{31.149} &0.97052\\
 			25.5  &28.584 &0.93979 &26.310 &0.92156   &29.999 &0.94919 &\textbf{30.847} &\textbf{0.95140}\\
 			\hline
 		\end{tabular}
 	\end{table}
 \end{center}

\begin{figure}[H]
	\centering
	\makebox[0pt]{
		\subfigure[]{\label{fig:gull}\includegraphics[width=0.3\textwidth]{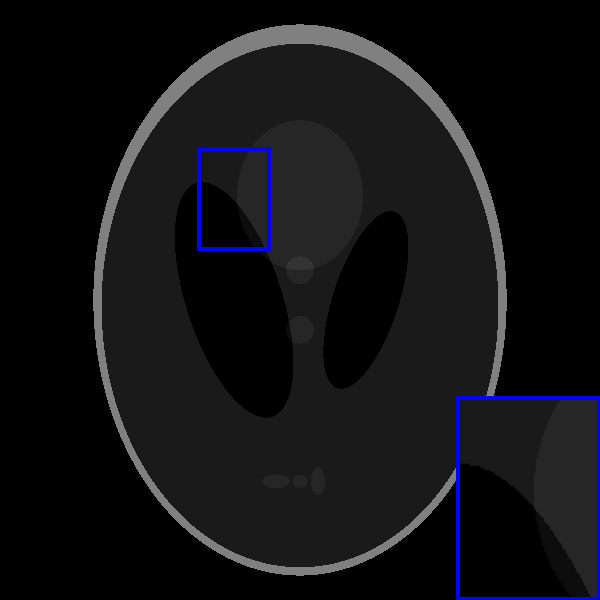}}
		\subfigure[]{\label{fig:tiger}\includegraphics[width=0.3\textwidth]{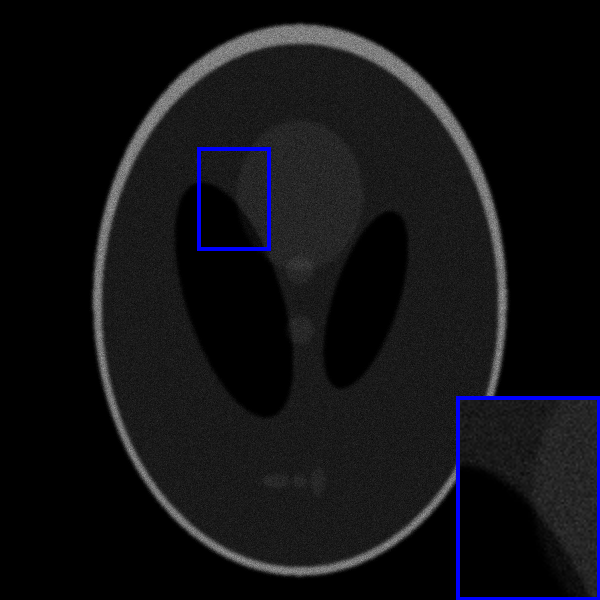}}
		\subfigure[]{\label{fig:mouse}\includegraphics[width=0.3\textwidth]{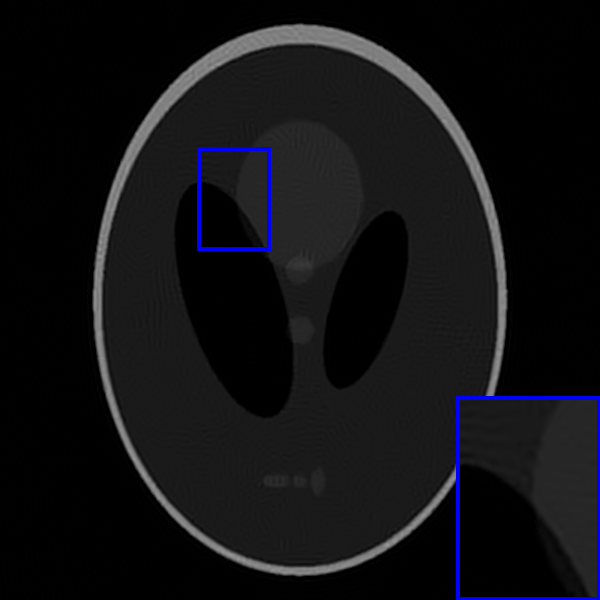}}}
	\\
	\vspace{-.3cm}
	\makebox[0pt]{
		\subfigure[]{\label{fig:gull}\includegraphics[width=0.3\textwidth]{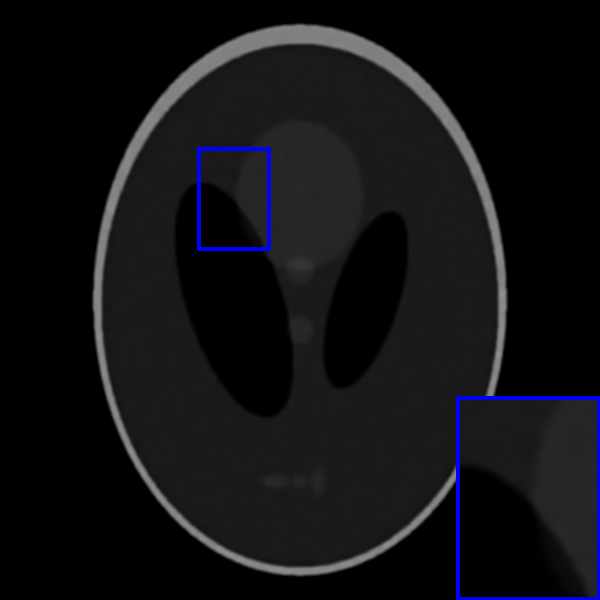}}
		\subfigure[]{\label{fig:tiger}\includegraphics[width=0.3\textwidth]{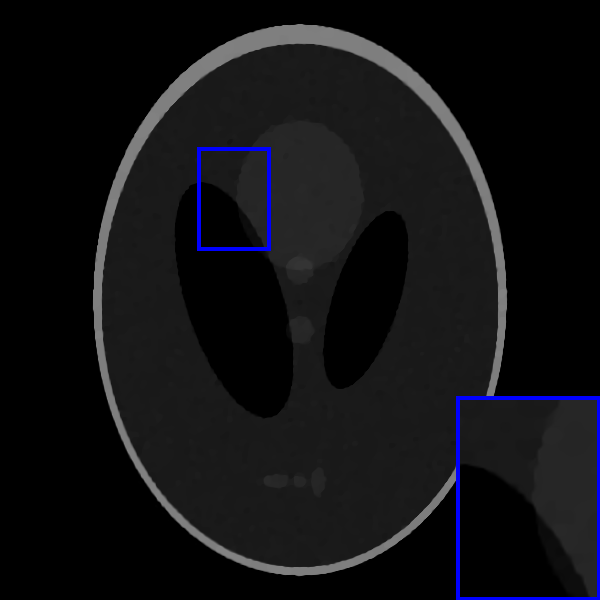}}
		\subfigure[]{\label{fig:mouse}\includegraphics[width=0.3\textwidth]{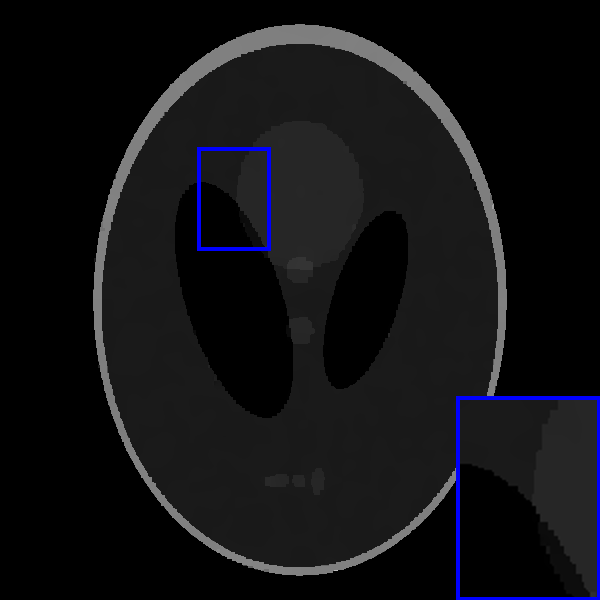}}}
	\vspace{-.2cm}
	\caption{Visual comparison of Shepp-Logan Phantom images with peak 127.5: (a) clear input image, 
		(b) blurred and noisy image, (c) BM3D \cite{27}, 
		(d) OGS \cite{28}, (e) FOTV \cite{29}, and (f) proposed method.}
	\label{figEx2}
\end{figure}

As mentioned previously, Poisson noise affects images acquired from telescopes, 
causing the loss of fine details in the images. 
To simulate this scenario, we selected the image \texttt{5.1.09}-Moon surface ($256 \times 256$) from 
the USC-SIPI Image Database. 
We then blurred it using a motion PSF with a motion length of 
 $15$ and an angle of $45^\circ$, and subsequently 
added Poisson noise to it.
The results related to this image are presented in Table  \ref{tabEx3} and Figure \ref{figEx3}. 
As can be observed in Figure \ref{figEx3}, a portion of the lunar mountains
has been lost due to noise and blurring processes and is no longer 
visible. However, the image restored by the proposed method 
has successfully recovered this part of the image.
Considering the three examples examined so far and comparing the numerical results and restored images, 
it can be seen that the proposed algorithm has an acceptable performance in image restoration.

\begin{center}
	\begin{table}[H]
		\caption{\small{Compared results of	PSNR and MSSIM for different
				peaks of \texttt{5.1.09}-Moon surface ($256 \times 256$).}}
		\label{tabEx3}
		\centering
		\small
		\begin{tabular}{ccccccccccccccc}
			\hline
			& \multicolumn{2}{c}{BM3D \cite{27}} &\multicolumn{2}{c}{OGS \cite{28}}  &\multicolumn{2}{c}{FOTV \cite{29}}
			&\multicolumn{2}{c}{Proposed method} \\
			\cmidrule(l){2-3} \cmidrule(l){4-5} \cmidrule(l){6-7}  \cmidrule(l){8-9} 
			Peak  &PSNR &MSSIM &PSNR &MSSIM &PSNR &MSSIM&PSNR &MSSIM\\
			\cmidrule(l){1-1} 		\cmidrule(l){2-3} \cmidrule(l){4-5} \cmidrule(l){6-7}  \cmidrule(l){8-9} 
			255   &27.889 &0.55673 &26.187 &0.54818 &28.506 &\textbf{0.58312} &\textbf{28.643}&0.57825\\
			127.5 &27.084 &0.54901 &25.467 &0.53996 &27.402 &0.55833 &\textbf{28.233}&\textbf{0.55945}\\
			51    &27.609 &0.52627 &23.875 &0.51026 &25.556 &0.52068 &\textbf{30.962} &\textbf{0.53581}\\
			25.5  &26.205 &0.48067 &23.710 &0.46878 &24.985 &0.47692 &\textbf{27.743} &\textbf{0.52453}\\
			\hline
		\end{tabular}
	\end{table}
\end{center}

\begin{figure}[H]
	\centering
	\makebox[0pt]{
		\subfigure[]{\label{fig:gull}\includegraphics[width=0.3\textwidth]{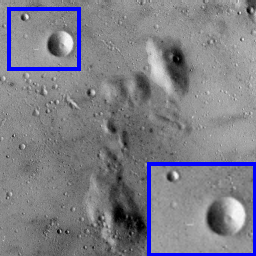}}
		\subfigure[]{\label{fig:tiger}\includegraphics[width=0.3\textwidth]{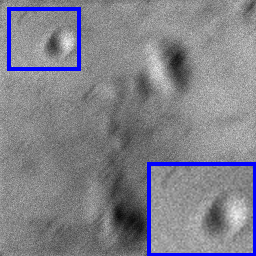}}
		\subfigure[]{\label{fig:mouse}\includegraphics[width=0.3\textwidth]{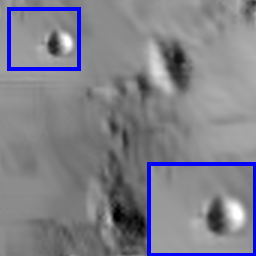}}}
	\\
	\vspace{-.3cm}
	\makebox[0pt]{
		\subfigure[]{\label{fig:gull}\includegraphics[width=0.3\textwidth]{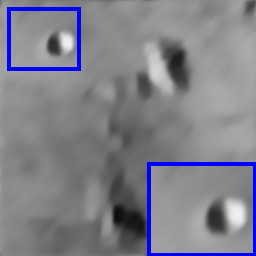}}
		\subfigure[]{\label{fig:tiger}\includegraphics[width=0.3\textwidth]{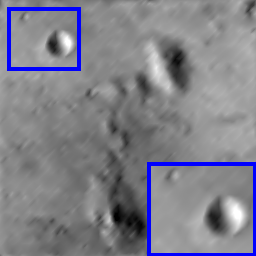}}
		\subfigure[]{\label{fig:mouse}\includegraphics[width=0.3\textwidth]{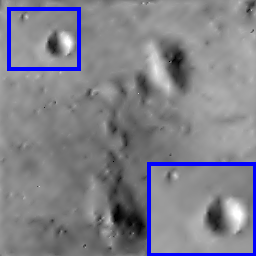}}}
	\vspace{-.2cm}
	\caption{Visual comparison of \texttt{5.1.09}-Moon surface ($256 \times 256$) images with peak 256: (a) clear input image, 
		(b) blurred and noisy image, (c) BM3D \cite{27}, 
		(d) OGS \cite{28}, (e) FOTV \cite{29}, and (f) proposed method.}
	\label{figEx3}
\end{figure}

In the previous sections, the convergence of the method has been theoretically investigated. 
In this section, to numerically evaluate the convergence of the method, the values of PSNR 
and Error for the three previous examples are plotted in Figures \ref{figError}  and \ref{figPsnr} for different iterations. 
As can be observed from the plots, despite the presence of fluctuations in some peaks, in general, 
the error gradually decreases and the PSNR value approaches a constant number.

\begin{figure}[H]
	\centering
	\makebox[0pt]{
		\subfigure[]{\label{fig:gull}\includegraphics[width=0.42\textwidth]{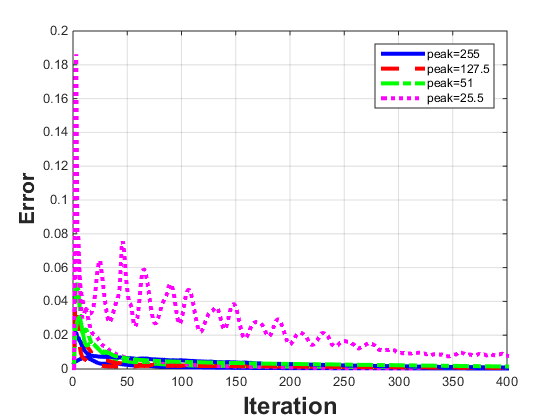}}
		\subfigure[ ]{\label{fig:tiger}\includegraphics[width=0.42\textwidth]{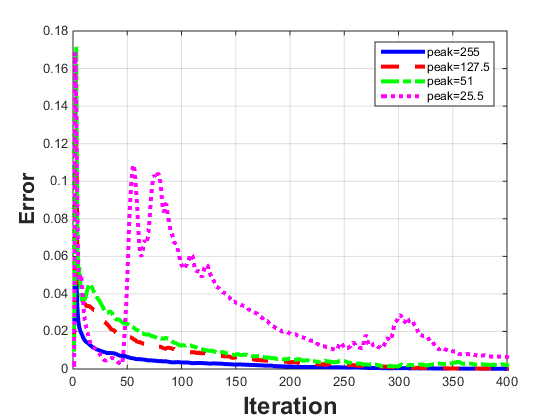}}
		\subfigure[]{\label{fig:mouse}\includegraphics[width=0.42\textwidth]{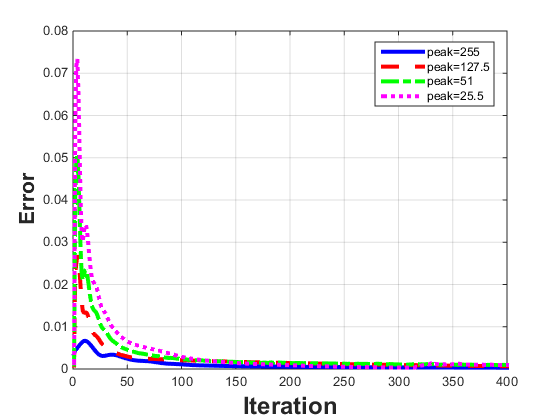}}}
		\vspace{-.2cm}
	\caption{ Error curves for test images with different peaks: (a) Girl,  (b) Shepp-Logan Phantom image, (c) \texttt{5.1.09}-Moon surface image.}
	\label{figError}
\end{figure}

\begin{figure}[H]
	\centering
	\makebox[0pt]{
		\subfigure[]{\label{fig:gull}\includegraphics[width=0.42\textwidth]{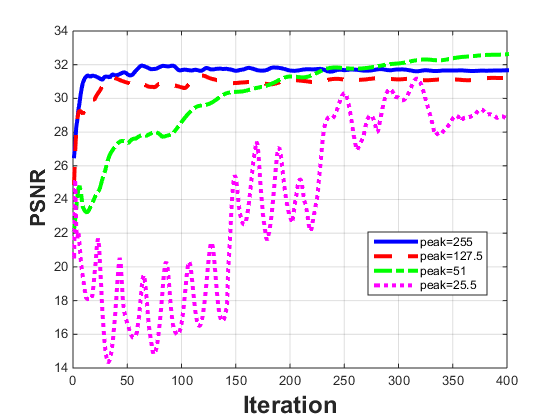}}
		\subfigure[]{\label{fig:tiger}\includegraphics[width=0.42\textwidth]{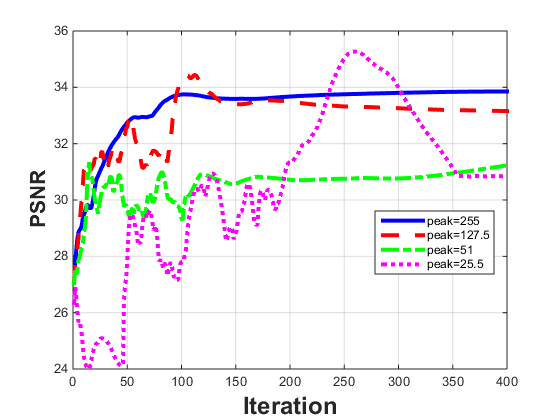}}
		\subfigure[]{\label{fig:mouse}\includegraphics[width=0.42\textwidth]{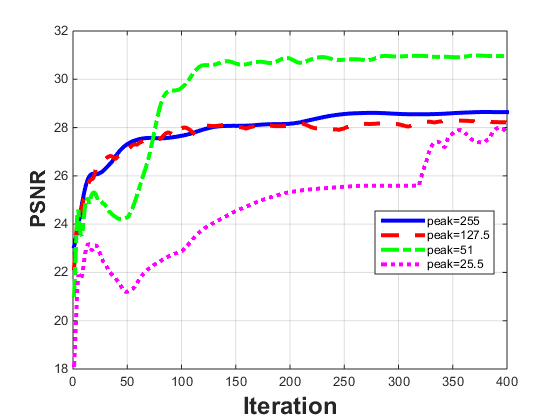}}}
	\vspace{-.2cm}
	\caption{PSNR curves for test images with different peaks: (a) Girl,  (b) Shepp-Logan Phantom image, (c) \texttt{5.1.09}-Moon surface image.}
	\label{figPsnr}
\end{figure}

The TV based methods usually have several parameters, and these parameters change the solution 
by adding the cost term for each regularizer in the objective function. In this proposed model, 
we consider the three main parameters as: $\beta$ for the derivative order, 
$\lambda$ and $\mu$ for regularization parameters. 
To show the effect of these parameters, and to analyze the sensitivity of the output, 
we plot the amount of change for these parameters for the PSNR value.
In this case, Hill image is considered as test image and $\sigma_i=1.01,~i=1,\cdots,4$, $\gamma_1=0.5,\gamma_2=\gamma_3=0.01$ and
$\gamma_4=0.001$ are used as input parameters. We then changed the parameters  $\beta, \mu$ and $\lambda$. 
The amount of change proportional to these parameters is shown in Figure \ref{fig7}. 
According to the results, it can be seen that the choice of parameter affects the output.
In the  proposed algorithm, we use grid search over the domain to find 
the appropriate parameters.

\begin{figure}[H]
	\centering
	\makebox[0pt]{
		\subfigure[]{\label{fig:gull}\includegraphics[width=0.42\textwidth]{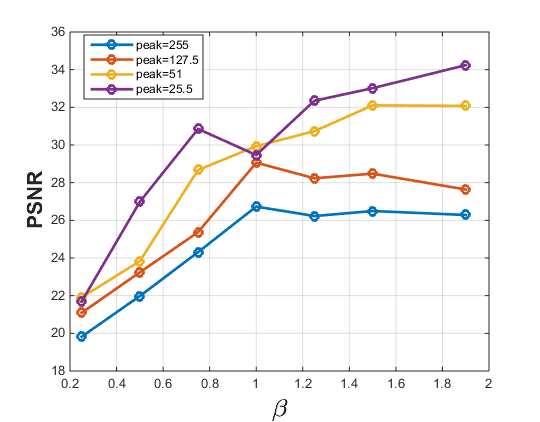}}
		\subfigure[]{\label{fig:tiger}\includegraphics[width=0.42\textwidth]{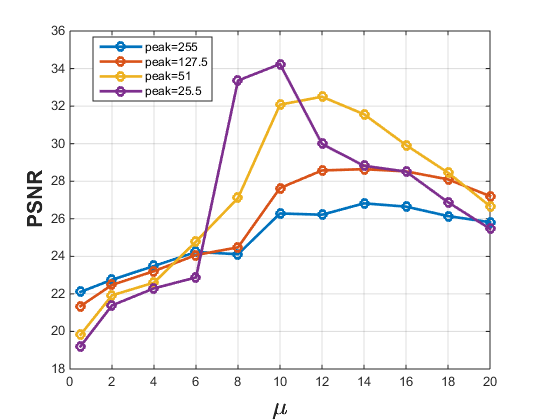}}
		\subfigure[]{\label{fig:mouse}\includegraphics[width=0.42\textwidth]{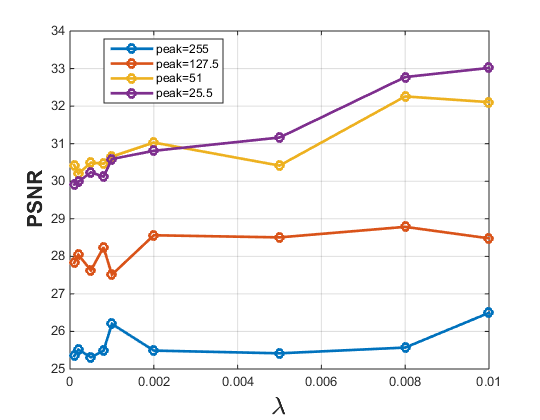}}}
	\vspace{-.2cm}
	\caption{Sensitivity analysis of (a) $\beta$ parameter, (b) $\mu$ parameter, (c) $\lambda$ parameter.}
	\label{fig7}
\end{figure}

In the previous examples we have studied nonblind case where we have information about the structure of 
the PSF. In this example, we study the blind case. For this purpose, we  study the 
Satellite ($128 \times 128$) image with a peak $1e03$ and  the following kernels: a motion PSF with a length of 
$15$ and an angle of $45^\circ$ (kernel 1)  and a Gaussian PSF  of
size ($7 \times 7$) with a standard deviation of 
$10$ (kernel 2). 
The results of the proposed blind algorithm and its comparison with  the EM  and FOTV
algorithms are given in Table \ref{tabEx4}. Also, the restored image and the 
calculated PSF are given in Figure \ref{figEx4}. According to the results, 
it can be seen that the proposed method has an acceptable performance in restoring both the image and the PSF.
To check the convergence of the proposed method, the changes in the error rate for the image and
 the obtained kernel are presented in Figure \ref{figEx4}. This figure also shows the rate of 
 change of the PSNR. As can be observed from these figures, as the number of iterations of 
 the proposed method increases, the error rate gradually decreases, and the PSNR value approaches a constant number.
\begin{center}
	\begin{table}[H]
		\caption{\small{Compared results of	PSNR and MSSIM for different
				kernels of Satellite ($128 \times 128$).}}
		\label{tabEx4}
		\centering
		\small
		\begin{tabular}{ccccccccccccccc}
			\hline
			& \multicolumn{2}{c}{Kernel 1} &\multicolumn{2}{c}{Kernel 2}  \\
			\cmidrule(l){2-3} \cmidrule(l){4-5} \cmidrule(l){6-7}  \cmidrule(l){8-9} 
			 Method &PSNR &MSSIM &PSNR &MSSIM\\
			\cmidrule(l){1-1} 		\cmidrule(l){2-3} \cmidrule(l){4-5}  
			EM \cite{29}  &20.5247&0.8070&22.646&0.8296 \\
			FOTV \cite{29} &24.7858&\textbf{0.8975}&23.952 &0.8578 \\
			Proposed     &\textbf{26.6419} &0.8920&\textbf{25.094} & \textbf{0.8593}\\
			\hline
		\end{tabular}
	\end{table}
\end{center}

\begin{figure}[H]
	\centering
	\makebox[0pt]{
		\subfigure[]{\label{fig:tiger}\includegraphics[width=0.25\textwidth]{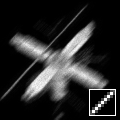}}
		\subfigure[]{\label{fig:mouse}\includegraphics[width=0.25\textwidth]{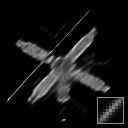}}
			\subfigure[]{\label{fig:mouse}\includegraphics[width=0.25\textwidth]{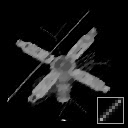}}
					\subfigure[]{\label{fig:mouse}\includegraphics[width=0.25\textwidth]{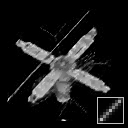}}}
	\\
	\vspace{-.3cm}
	\makebox[0pt]{
		\subfigure[]{\label{fig:tiger}\includegraphics[width=0.25\textwidth]{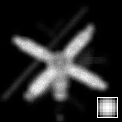}}
\subfigure[]{\label{fig:mouse}\includegraphics[width=0.25\textwidth]{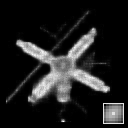}}
\subfigure[]{\label{fig:mouse}\includegraphics[width=0.25\textwidth]{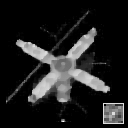}}
\subfigure[]{\label{fig:mouse}\includegraphics[width=0.25\textwidth]{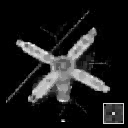}}}
	\vspace{-.2cm}
	\caption{Visual comparison of Satellite ($128 \times 128$) image with peak 1e03: (a), (e) blurred and noisy versions using Kernel 1 and Kernel 2,respectively;
		(b),(f) restored images by EM \cite{29}; (c),(g) restored images by FOTV \cite{29}; (d),(h)  restored images by proposed method.
}
	\label{figEx4}
\end{figure}

\begin{figure}[H]
	\centering
	\makebox[0pt]{
		\subfigure[]{\label{fig:gull}\includegraphics[width=0.42\textwidth]{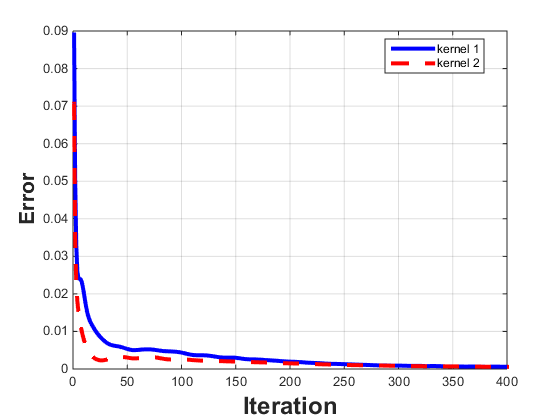}}
		\subfigure[]{\label{fig:tiger}\includegraphics[width=0.42\textwidth]{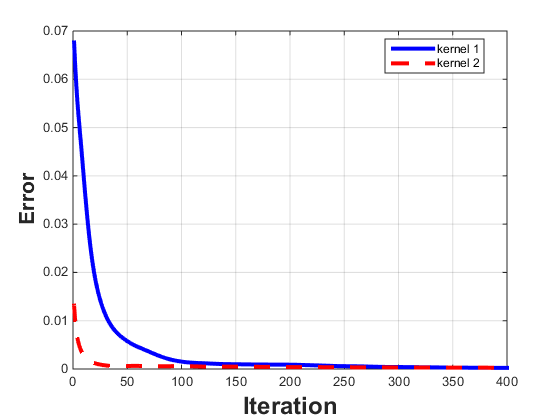}}
		\subfigure[]{\label{fig:mouse}\includegraphics[width=0.42\textwidth]{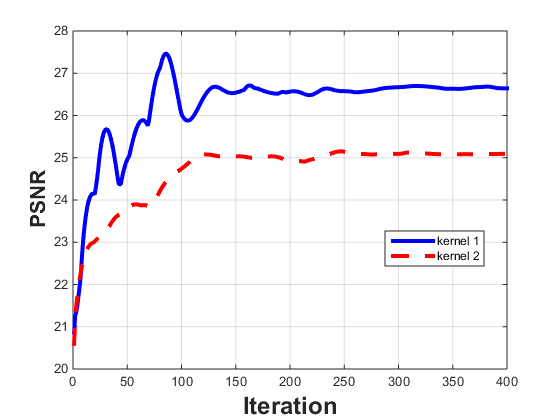}}}
	\vspace{-.2cm}
	\caption{ Error and PSNR curves results: (a) Error curves for restored image, (b) Error curves for restored PSF,
		 (c) PSNR curves for restored image.}
	\label{fig8}
\end{figure}

\section{Conclusion}\label{sec5}

In this paper, we have proposed a novel model based on  the minimax concave penalty and the reweighted
 $\ell_1$ norm for Poisson image denoising. The model is further extended to handle 
blind deconvolution problems, and an ADMM based algorithm is used for its 
numerical solution.We provide theoretical convergence analysis and validate the method through numerical experiments.
The obtained results confirm the  
performance of the proposed method in both accuracy and convergence behavior.

\end{document}